\documentclass[10pt]{article}
\usepackage[utf8]{inputenc}
\usepackage{amsmath, amssymb, amsthm}
\allowdisplaybreaks
\usepackage{graphicx}
\usepackage{hyperref}
\usepackage{geometry}
\usepackage{times}  
\usepackage{natbib}  
\usepackage{authblk}  
\usepackage{anyfontsize} 

\usepackage{array}
\usepackage{tabu}
\usepackage{wrapfig}
\usepackage{float}
\usepackage{thm-restate}
\usepackage{graphicx}
\usepackage{subcaption}
\usepackage{hyperref}
\usepackage{url}
\usepackage{xcolor}
\usepackage{multirow}
\usepackage{booktabs}
\usepackage[ruled,vlined]{algorithm2e}
\usepackage{bm}
\usepackage{mathrsfs}

\DeclareMathOperator*{\essinf}{ess\,inf}
\DeclareMathOperator*{\esssup}{ess\,sup}

\theoremstyle{plain}
\newtheorem{theorem}{Theorem}[section]
\newtheorem{proposition}[theorem]{Proposition}
\newtheorem{lemma}[theorem]{Lemma}

\theoremstyle{definition}

\newtheorem{assumption}[theorem]{Assumption}
\theoremstyle{remark}

\newcommand{\EDmodified}[1]{{#1}}

\author[1,2]{Yudong Luo}
\author[1,2]{Erick Delage}

\affil[1]{GERAD \& Department of Decision Sciences, HEC Montréal, Canada}\affil[2]{Mila - Quebec AI Institute, Canada}

\title{Actor-Critic Algorithm for Dynamic Expectile and CVaR}

\date{}

\begin{document}
\maketitle

\begin{abstract}
Optimizing dynamic risk with stochastic policies is challenging in both policy updates and value learning. The former typically requires transition perturbation, while the latter may rely on model-based approaches. To address these challenges, we propose a surrogate policy gradient without transition perturbation under softmax policy parameterization. We further develop model-free value learning methods for dynamic expectile and conditional value-at-risk by leveraging elicitability. Finally, inspired by Expected SARSA and Expected Policy Gradient, a model-free off-policy actor–critic algorithm is constructed. Empirical results in domains with verifiable risk-averse behavior show that our algorithm can learn risk-averse policy and consistently outperforms other existing methods.
\end{abstract}

\section{Introduction}
The demand for risk control in sequential decision-making has inspired risk-averse reinforcement learning (RL). Generally, there exist two main types of risk-sensitive objectives in the literature. One defines risk over the total return variable and is referred to as static risk. The other defines the risk recursively along the trajectory and is referred to as dynamic risk. While static risk has advantages in interpretability, optimizing a policy under static risk can be difficult, as for most of them, the optimal policy is history-dependent in the original Markov decision process (MDP) (except mean, essential supremum, and essential infimum)~\citep{benyamine2026dynamic}. In contrast, dynamic risk models the behavior of a decision maker who accounts for risk at each step. Its recursive structure also leads to a Bellman-type equation, which implies the optimal policy is Markovian~\citep{ruszczynski2010risk}.

Although the Bellman equation for dynamic risk is analogous to the risk-neutral case, developing a practical actor-critic algorithm raises challenges for two reasons. 1) For coherent risk measures~\citep{artzner1999coherent}, computing policy gradients with stochastic policy requires perturbing the transition probabilities according to its risk envelope~\citep{tamar2015policy,zhang2024soft}, which significantly increases complexity. 2) The value function is itself a risk measure over the randomness induced by the transition dynamics, which is usually hard to obtain. Extracting the risk value directly from the value distribution learned by distributional RL~\citep{dabney2018distributional} does not guarantee correctness~\citep{lim2022distributional}. As a result, most existing works may either adopt model-based approaches for value learning (e.g.,\cite{tamar2015policy,rigter2023one}) or resort to deterministic policy gradient~\citep{silver2014deterministic} to simplify policy updates (e.g.,\cite{marzban2023deep}).

We focus on two widely used coherent risk measures for dynamic risk, namely expectile and conditional value-at-risk (CVaR), and aim to develop model-free actor-critic algorithms with stochastic policies for them. First, from a policy iteration perspective, we show that a stochastic policy updated by a surrogate policy gradient, without perturbing transition probabilities, converges to the optimal policy with proper learning rate, when using the popular softmax parameterization. Next, leveraging the elicitability of expectile~\citep{marzban2023deep} and motivated by the joint elicitability of CVaR~\citep{coache2023conditionally}, we derive value learning rules based on stochastic optimization, analogous to Q-learning in the risk-neutral setting. Finally, inspired by Expected SARSA~\citep{sutton1998reinforcement} and Expected Policy Gradient~\citep{ciosek2020expected}, an off-policy actor-critic algorithm is formalized. To the best of our knowledge, this is the first model-free off-policy actor–critic framework with stochastic policies for dynamic expectile and CVaR, and the idea of policy improvement also applies to other dynamic risk measures. Empirical results in domains with verifiable risk-averse behavior demonstrate that our methods can learn risk-averse policy when others may fail.

\section{Preliminary}
\subsection{Risk Measures, Expectile, and CVaR}
\label{sec:prelim}
For a real valued random variable $\tilde{x}$ (we use tilde to mark a random variable, e.g., $\tilde{x}$, in this paper), a \textit{risk measure} is a function $\rho: \tilde{x}\rightarrow \mathbb{R}\cup\{-\infty,+\infty\}$ that maps an uncertain return to a quantity that captures the risk adjusted value of $\tilde{x}$. An important class is \textit{coherent} risk measure, that possesses a set of properties that are consistent with rational risk assessments~\citep{artzner1999coherent}, i.e.,

\textbf{A1} Concavity: $\forall \lambda\in[0,1],\rho(\lambda\tilde{x}+(1-\lambda)\tilde{y})\geq \lambda \rho(\tilde{x}) + (1-\lambda)\rho(\tilde{y})$;

\textbf{A2} Monotonicity: if $\tilde{x}\leq \tilde{y}$, then $\rho(\tilde{x})\leq \rho(\tilde{y})$;

\textbf{A3} Translation invariance: $\forall c\in\mathbb{R}$, $\rho(\tilde{x}+c)=\rho(\tilde{x})+c$;

\textbf{A4} Positive homogeneity: if $\lambda\geq 0$, then $\rho(\lambda \tilde{x})=\lambda \rho(\tilde{x})$.

Superadditivity (i.e., $\rho(\tilde{x}+\tilde{y})\geq \rho(\tilde{x})+\rho(\tilde{y})$) can be inferred from (A1) and (A4), which may appear in the definition of coherence in other papers. Note that since we work on return instead of cost, some properties, i.e. Concavity and Superadditivity, are adjusted accordingly. Here, we consider two common coherent risk measures, i.e., expectile and CVaR. 

For a risk level $\alpha\in(0,1)$, the $\alpha$-expectile is defined as~\citep{bellini2017risk}
\vspace{-0.02in}
\begin{equation}
\label{eq:def-expectile}
    \mathrm{Expectile}_\alpha[\tilde{x}] := \arg\min_{y\in\mathbb{R}} (1-\alpha)\mathbb{E}\big[(\tilde{x}-y)^2_{-}\big]+\alpha \mathbb{E}\big[(\tilde{x}-y)^2_{+}\big],
\end{equation}
\vspace{-0.025in}
where $(x)_{-}:=\min(x,0)$ and $(x)_{+}:= \max(x,0)$. Expectile covers the case of risk-neutrality (i.e., expectation) when $\alpha=1/2$. As $\alpha\rightarrow 0$ or $\alpha\rightarrow 1$, it recovers $\essinf(\tilde{x})$ or $\esssup(\tilde{x})$, respectively.

The left-tail CVaR is defined as~\citep{rockafellar2000optimization}
\vspace{-0.05in}
\begin{equation}
\label{eq:def-cvar}
    \mathrm{CVaR}_\alpha[\tilde{x}]:=\max_{y\in\mathbb{R}} y -\frac{1}{\alpha}\mathbb{E}[(y-\tilde{x})_+],
\end{equation}
\vspace{-0.02in}
where the maximum is always attained at $y$ equals the $\alpha$-value-at-risk (VaR) of $\tilde{x}$ as a by product. The $\alpha$-VaR is defined as the largest $1-\alpha$ confidence lower bound on the value of $\tilde{x}$, i.e., $\mathrm{VaR}_\alpha[\tilde{x}]:=\max\{x\in\mathbb{R}|P[\tilde{x}< x]\le \alpha\}$. For ease of presentation, we assume the random variable is continuous and its quantile is unique in this paper. In this case, VaR coincides with quantile, and CVaR can be alternatively expressed as
\vspace{-0.05in}
\begin{equation*}
    \mathrm{CVaR}_\alpha[\tilde{x}]:=\frac{1}{\alpha}\int_0^\alpha \mathrm{VaR}_\beta[\tilde{x}] d\beta~.
\end{equation*}
Thus, $\mathrm{CVaR}_\alpha[\tilde{x}]$ can be interpreted as the expected value of the bottom $\alpha$ portion of tail quantiles, i.e., $\mathrm{CVaR}_\alpha[\tilde{x}]=\mathbb{E}[\tilde{x}|\tilde{x}\leq \mathrm{VaR}_\alpha[\tilde{x}]]$. As $\alpha\rightarrow 0$ or $\alpha\rightarrow 1$, it recovers $\essinf(\tilde{x})$ or $\mathbb{E}[\tilde{x}]$, respectively.

\subsection{(Coherent) Dynamic Risk in RL}
\label{sec:dynamic-risk}
In standard RL settings, the interaction between the agent and the environment is modeled as a MDP, represented by a tuple $\langle \mathcal{S},\mathcal{A},R,P,\mu_0,\gamma\rangle$. Here, $\mathcal{S}$ and $\mathcal{A}$ denote state and action spaces. $R(s,a,\tilde{\omega})$ is the state and action dependent reward function with randomness $\tilde{\omega}$. $P(\cdot|s,a)$ is the transition function. $\mu_0$ is the initial state distribution, and $\gamma\in(0,1)$ is the discount factor. For notational simplicity, we use $\tilde{r}(s,a)$ to denote the random reward variable associated with the pair $(s,a)$ and $r(s,a)$ to denote a realization. When the context is clear, we may drop the $(s,a)$ argument and simply use $\tilde{r}$ or $r$. The agent collects a sequence of state $\tilde{s}_t$, action $\tilde{a}_t$, and reward $\tilde{r}_t$ by executing its (Markovian) policy $\pi$ in the MDP. Without loss of generality, we consider $s_0$ is fixed when defining the objective. For simplicity, we begin from the case where $\pi$ is deterministic to introduce policy learning of dynamic risk, followed by the discussion for the case where $\pi$ is stochastic. We overload the notation for $\pi$, so that $\pi(s)$ refers to a deterministic policy and $\pi(\cdot|s)$ refers to a stochastic one.
\subsubsection{Deterministic Policy Case}
\label{sec:deter-policy-case}
For a policy $\pi$, MDP $M$, coherent risk measure $\rho$ (e.g.,$\alpha$-expectile or $\alpha$-CVaR), the dynamic risk objective $\rho^\pi_\infty(M)$ is defined as~\citep{rigter2023one}
\begin{equation}
\label{eq:dynamic-risk-def}
    \rho^\pi_\infty(M)=\rho\Big(\tilde{r}(s_0,\pi(s_0))+\gamma \rho\big(\tilde{r}(\tilde{s}_1,\pi(\tilde{s}_1))+\gamma \rho(\tilde{r}(\tilde{s}_2,\pi(\tilde{s}_2))+...) \big)\Big).
\end{equation}
Define the action-value function under policy $\pi$ as $Q^\pi(s,a)=\rho^\pi_\infty(M|s_0=s, a_0=a)$, and the state-value function as $V^\pi(s)=\rho^\pi_\infty(M|s_0=s)$ (it holds that $V^\pi(s)=Q^\pi(s,\pi(s))$). The Bellman equation is given by~\citep{ruszczynski2010risk}
\begin{equation}
\label{eq:bellman-eq-pi}
    Q^\pi(s,a) =\rho\big(\tilde{r}(s,a)+\gamma Q^\pi(\tilde{s}',\pi(\tilde{s}'))\big).
\end{equation}
Solving for $\max_\pi \rho^\pi_\infty(M)$ admits an optimal value function $Q^*$, where $Q^*(s,a)=\max_\pi Q^\pi(s,a)$, $V^*(s)=\max_aQ^*(s,a)$, and the corresponding optimal policy is $\pi^*(s)=\arg\max_a Q^*(s,a)$. The Bellman optimality equation is~\citep{ruszczynski2010risk}
\begin{equation}
\label{eq:bellman-opt-eq}
    Q^*(s,a)=\rho\big(\tilde{r}(s,a)+\gamma\cdot\max_{a'} Q^*(\tilde{s}',a')\big)
\end{equation}
When $\rho$ is $\mathbb{E}$, these two equations reduce to the standard Bellman equations in risk-neutral setting. For general (non-linear) risk measures, computing $\rho$ in a closed is often difficult. A straightforward approach is to learn environment dynamics (i.e., model-based) to approximate the distribution of $\tilde{r}$ and $\tilde{s}'$, so as to assist computing $\rho$ (e.g., \cite{rigter2023one} for dynamic CVaR). For clarity of presentation, we defer the discussion of model-free value learning for dynamic expectile and CVaR to Sec.~\ref{sec:learn-value}. For now, we assume $Q^\pi(s,a)$ can be learned via Eq.~\ref{eq:bellman-eq-pi} for any $\pi$.

Though Eq.~\ref{eq:bellman-opt-eq} suggests a value-based approach similar to Q-learning~\citep{sutton1998reinforcement}, actor-critic methods are often preferred for better stability and generalization. Let $\pi$ be parameterized by $\theta$. When $\pi$ is deterministic, the construction of an actor-critic is similar to deterministic policy gradient (DPG)~\citep{silver2014deterministic} in risk-neutral case (e.g.,\cite{marzban2023deep} for dynamic expectile). The policy update takes the form $\int_{\mathcal{S}} d(s)\nabla_\theta \pi_\theta(s) \nabla_aQ^\pi(s,a)|_{a=\pi_\theta(s)} ds$, where $d(s)$ is a state distribution. However, DPG is mainly designed for continuous action spaces, which limits its applicability.
\subsubsection{Stochastic Policy Case}
\label{sec:st-case}
For stochastic policies, there exist two ways to define the risk objective as in Eq.~\ref{eq:dynamic-risk-def}, depending on how the randomness induced by the policy is handled. One approach incorporates the randomness of $\pi(\cdot|s)$ directly into $\rho$, so that the risk is taken over the joint distribution of $\tilde{s}$, $\tilde{a}$, and $\tilde{r}$~\citep{tamar2015policy,tamar2016sequential,huang2021convergence}, i.e., $\rho^\pi_\infty(M)=\rho\big(\tilde{r}(s_0,\tilde{a}_0)+\gamma \rho(\tilde{r}(\tilde{s}_1,\tilde{a}_1)+...)\big)$. The other approach marginalizes the policy randomness by taking expectation~\citep{shen2013risk,yu2023global,zhang2024soft}, i.e., $\rho_\infty^\pi(M)=\mathbb{E}_{a\sim \pi(\cdot|s_0)}\big[\rho\big(\tilde{r}(s_0,a_0)+\gamma \mathbb{E}_{a_1\sim\pi(\cdot|\tilde{s}_1)}[\rho(\tilde{r}(\tilde{s}_1,a_1)+...)]\big)\big]$, with unique fixed point $Q^\pi$ and with $V^\pi(s)=\mathbb{E}_{a\sim\pi(\cdot|s)}[Q^\pi(s,a)]$. In this work, we adopt the latter definition for two reasons. 1) Including policy stochasticity in the risk measure treats exploration as a source of risk, which may hinder exploration in risk-averse setting. 2) Since there exists optimal deterministic policy (Eq.~\ref{eq:bellman-opt-eq}), a stochastic policy is expected to converge toward a deterministic one. Accordingly, we consider the following Bellman operator for stochastic policies, similar to \cite{zhang2024soft}
\begin{equation}
\label{eq:bellman-op-st}
    \mathcal{T}^\pi Q(s,a)=\rho\big(\tilde{r}(s,a)+\gamma \sum_{a'}\pi(a'|\tilde{s}')Q(\tilde{s}',a')\big)
\end{equation}
Next, consider the policy update. By the dual representation of coherent risk measures, $\rho(\cdot)$ can be expressed as an expectation under a perturbed distribution. This allows one to extend stochastic policy gradient (SPG) methods~\citep{sutton1998reinforcement} from the risk-neutral setting to dynamic risk. Specifically, the dual form is $\rho(\tilde{x})=\inf_{\xi}\mathbb{E}_\xi[\tilde{x}]$, where $\xi$ is  the risk envelope satisfying $\int_x\xi(x)f_{\tilde{x}}(x)=1$ and $\xi(x)\geq0$. The additional constraints on $\xi$ depend on the specific choice of $\rho$. In Eq.~\ref{eq:bellman-op-st}, the randomness comes jointly from $\tilde{r}$ and $\tilde{s}'$. As a result, $\xi$ perturbs the transition distribution $P(r,s'|s,a)$ to $\xi(r,s'|s,a)P(r,s'|s,a)$. Let $\xi^*$ denote the optimal risk envelope and $P^{\xi^*}$ the corresponding perturbed transition. The policy gradient is then given by~\citep{zhang2024soft}
\begin{equation}
\label{eq:dynamic-risk-pg}
    \nabla_\theta V^{\pi_\theta}(s_0) = \mathbb{E}_{a_0\sim\pi_\theta(\cdot|s_0)}\mathbb{E}_{\xi^*}\Big[\sum_{t=0}^\infty\gamma^tQ^{\pi_\theta}(s_t,a_t)\nabla_\theta\log\pi_\theta(a_t|s_t)\Big],
\end{equation}
where $\mathbb{E}_{\xi^*}$ denotes that the trajectory $\{s_t,a_t,r_t\}$ follows the perturbed transition $P^{\xi^*}$. The derivation is provided in Appendix~\ref{app:spg}. Note that when the risk measure also accounts for policy randomness, the corresponding policy gradient similarly involves transition perturbation (e.g., see~\cite{tamar2015policy}). However, computing $\xi^*$ at each step is non-trivial, as it typically requires solving a constrained optimization problem, often via Lagrangian methods~\citep{tamar2015policy}.

\section{Method}
As discussed in Sec.~\ref{sec:dynamic-risk}, applying SPG to coherent dynamic risk requires solving for the optimal transition perturbation at each step to generate trajectories, which is difficult to compute and implement in practice. Although DPG avoids this issue, it is limited to continuous action spaces. Therefore, a model-free actor-critic algorithm with stochastic policies is desired. In this section, we show that this can be achieved by using a surrogate policy gradient that avoids transition perturbation.

\subsection{A Surrogate Stochastic Policy Gradient}
\label{sec:learn-policy}
We first claim the assumptions we made for tabular MDP.
\begin{assumption}
\label{ass:1}
1) The space $\mathcal{S}\times\mathcal{A}$ is finite and $d(s)$ is a normalized state distribution satisfying $0<d(s)\leq 1$. 2) Reward $\tilde{r}\in[0,1]$. 3) The reward has bounded densities $\forall(s,a)$. 
\end{assumption}
With assumtion~\ref{ass:1}, the surrogate gradient we consider for policy parameter $\theta$ is
\begin{equation}
\label{eq:surrogate-pg}
    \theta_{t+1} \leftarrow \theta_t + \eta~\nabla_\theta\sum_s d(s)\sum_a \pi_{\theta_t}(a|s)Q^{\pi_t}(s,a),
\end{equation}
where $\theta_t$ is the policy parameter at iteration $t$, and $\eta$ is the learning rate. Here, $Q^{\pi_t}=Q^{\pi_{\theta_t}}$ is treated as fixed with respect to $\theta$, and is therefore omitted from taking gradient. We assume that $Q^{\pi_t}$ (Eq.~\ref{eq:bellman-op-st}) is available under $\pi_{\theta_t}$ when analyzing the behavior of policy update. This follows the standard policy iteration view, where the value function is fully evaluated before updating the policy. Eq.~\ref{eq:surrogate-pg} is analogous to SPG in risk-neutral setting but it uses risk-sensitive $Q$-function. Moreover, Eq.~\ref{eq:surrogate-pg} does not require transition perturbation as it only requires the state distribution $d(s)$ to have full coverage.

The motivation of proposing Eq.~\ref{eq:surrogate-pg} is based on the observation that when action is continuous and policy is parameterized by a Gaussian distribution, this surrogate gradient with respect to the policy mean reduces to DPG if the the policy variance goes to zero~\citep{ciosek2020expected}. Specifically, let $\pi_\theta(a|s)=\mathcal{N}\big(a;\mu_{\theta_1}(s),\Sigma_{\theta_2}(s)\big)$ with $\theta=\{\theta_1,\theta_2\}$, then
\begin{equation}
\label{eq:gaussian-mean-pg}
    \nabla_{\theta_1}\int_{\mathcal{S}} d(s)\int_{\mathcal{A}}\pi_\theta(a|s)Q^\pi(s,a) dads=\int_{\mathcal{S}}d(s)\nabla_{\theta_1}\mu_{\theta_1}(s)\nabla_aQ^\pi(s,a)|_{a=\mu_{\theta_1}(s)}ds, \mathrm{if}~\Sigma=0.
\end{equation}
The derivation of Eq.~\ref{eq:gaussian-mean-pg} is provided in Appendix~\ref{app:spg-gaussian}. This equivalence to DPG motivates us to investigate its behavior under discrete action space. Here, we are interested in the popular softmax policy parameterization when action is discrete, i.e., $\pi_\theta(a|s)=\exp\big(\theta(s,a)\big)/\sum_b\exp\big(\theta(s,b)\big)$. In this case, the update rule for $\theta_t(s,a)$ can be simplified as $\theta_{t+1}(s,a)\leftarrow \theta_t(s,a) + \beta~d(s)\pi_t(a|s)A^{\pi_t}(s,a)$, where $A^{\pi_t}(s,a)=Q^{\pi_t}(s,a)-V^{\pi_t}(s)$ (see Appendix~\ref{app:spg-softmax}). With softmax parameterization, Eq.~\ref{eq:surrogate-pg} (with proper learning rate) ensures the monotone increasing of value function and the convergence to the optimal value (therefore an optimal policy), as formally stated in Lemma~\ref{lemma:monotone} and Theorem~\ref{thm:optimal}.
\begin{restatable}{lemma}{Lemmamonotone}
\label{lemma:monotone}
    With assumption~\ref{ass:1}, for softmax policy, the update rule in Eq.~\ref{eq:surrogate-pg} with $\eta\leq(1-\gamma)/5$ ensures that for all $(s,a)\in\mathcal{S}\times\mathcal{A}$, $Q^{\pi_{t+1}}(s,a)\geq Q^{\pi_t}(s,a), ~V^{\pi_{t+1}}(s)\geq V^{\pi_t}(s)$.
\end{restatable}
\begin{restatable}{theorem}{Thmoptimal}
\label{thm:optimal}
   With assumption~\ref{ass:1} and the same conditions as in Lemma~\ref{lemma:monotone}, the update rule in Eq.~\ref{eq:surrogate-pg} ensures that for all $s$, $V^{\pi_t}(s)\rightarrow V^*(s)$ as $t\rightarrow\infty$.
\end{restatable}
The proofs of Lemma~\ref{lemma:monotone} and Theorem~\ref{thm:optimal} are provided in Appendix~\ref{app:monotone} and \ref{app:policy-converge}, which are similar to~\cite{agarwal2021theory} in risk neutral setting. The assumption that $\tilde{r}\in[0,1]$ is ease of presentation without loss of generality. For proving Lemma~\ref{lemma:monotone}, the key idea is that with proper learning rate $\eta$, $\mathbb{E}_{a\sim\pi_{t+1}(\cdot|s)}[Q^{\pi_t}(s,a)]\geq \mathbb{E}_{a\sim\pi_t(\cdot|s)}[Q^{\pi_t}(s,a)]$. By the monotonicity (A2) of $\rho$ in Eq.~\ref{eq:bellman-op-st}, this implies $\mathcal{T}^{\pi_{t+1}}Q^{\pi_t}\geq \mathcal{T}^{\pi_t}Q^{\pi_t}=Q^{\pi_t}$. Applying monotonicity again yields $Q^{\pi_{t+1}}=\lim_{k\rightarrow\infty}(\mathcal{T}^{\pi_{t+1}})^kQ^{\pi_t}\geq \mathcal{T}^{\pi_{t+1}}Q^{\pi_t}\geq Q^{\pi_t}$. Since the value function is pointwise increasing and bounded, the monotone convergence theorem ensures the existence of some $Q^\infty$ and $V^\infty$ such that $Q^{\pi_t}(s,a)\rightarrow Q^{\infty}(s,a)$ and $V^{\pi_t}(s)\rightarrow V^\infty(s)$. For Theorem~\ref{thm:optimal}, define the action sets $I^0_s:=\{a|Q^\infty(s,a)=V^\infty(s)\},I_s^+:=\{a|Q^\infty(s,a)>V^\infty(s)\},I_s^-:=\{a|Q^\infty(s,a)<V^\infty(s)\}$. The proof steps are to show $\pi_t(a|s)\rightarrow 0$ for actions $a\in I_s^+\cup I_s^-$ as $t\rightarrow \infty$, and $I_s^+$ is empty. The first part is based on the fact that the gradient tends to zero in the limit while $|A^{\pi_t}(s,a)|$ for $a\in I^+_s\cup I^-_s$ is non-zero. $I^+_s=\emptyset$ is obtained by contradiction. If $I^+_s$ is non-empty, then one can show that there is a subset set of actions in $I_s^0$ for which the summation of $\theta(s,a)$ goes to infinity while the summation of their updates is negative, thus leading to a contradiction.

The above analysis for both Gaussian and softmax policies suggests a model-free approach for policy updates based on Eq.~\ref{eq:surrogate-pg}, which avoids transition perturbation and bypasses the need to learn the optimal risk envelope, as required in many existing actor-critic methods~\citep{tamar2015policy,tamar2016sequential,huang2021convergence,zhang2024soft}. This surrogate gradient is central to the actor-critic algorithm proposed in this paper.

\subsection{Learning Value Function}
\label{sec:learn-value}
So far, we have assumed that $Q^\pi$ is available. In practice, learning $Q^\pi$ directly from Eq.~\ref{eq:bellman-op-st} is challenging when $\rho$ is non-linear. Based on the idea of elicitability of expectile, and motivated by the joint elicitability of CVaR, we show how to learn value function from transition samples for them.

\textbf{Dynamic Expectile.} A risk measure $\rho$ is elicitable if there exists a measurable function $L$ (called a loss function) such that $\rho(\tilde{x})=\arg\min_y \mathbb{E}[L(\tilde{x},y)]$~\citep{bellini2015elicitable}. Expectile is the only coherent risk measure that is elicitable, with its loss function given in Eq.~\ref{eq:def-expectile}. Therefore, when $\rho$ is $\alpha$-expectile, Eq.~\ref{eq:bellman-op-st} can be equivalently expressed as
\begin{equation*}
    \mathcal{T}^\pi Q(s,a)=\arg\min_y (1-\alpha)\mathbb{E}\big[(\tilde{r}+\gamma V(\tilde{s}')-y)^2_{-}\big] + \alpha \mathbb{E}\big[(\tilde{r}+V(\tilde{s}')-y)^2_{+}\big],
\end{equation*}
where $V(s')=\mathbb{E}_{a'\sim\pi(\cdot|s')}[Q(s',a')]$. This leads to a gradient-based Q-learning–style update, where $Q$ is updated along the negative gradient of the loss. Given a transition $(s,a,r,s')$ and learning rate $\zeta$, the update is
\begin{equation}
\label{eq:expectile-qlearning}
\begin{aligned}
    Q_{t+1}(s,a)\leftarrow &Q_t(s,a) - \zeta_t(s,a) \cdot \partial_y \big[(1-\alpha)(r+\gamma V_t(s')-y)^2_{-}+\alpha (r+\gamma V_t(s')-y)^2_{+}\big]\big|_{y=Q_t(s,a)}\\
    = &Q_t(s,a) + 2\zeta_t(s,a) \cdot\big[(1-\alpha)(r+\gamma V_t(s')-Q_t(s,a))_{-}+\alpha(r+\gamma V_t(s')-Q_t(s,a))_{+}\big].
\end{aligned}
\end{equation}
This model-free approach for learning value function under dynamic expectile has been studied in, e.g.,~\cite{marzban2023deep,jiang2024learning}.

\textbf{Dynamic CVaR.} CVaR is not elicitable, making it more challenging than expectile. However, it is jointly elicitable with VaR, i.e., $\big(\mathrm{VaR}_\alpha(\tilde{x}),\mathrm{CVaR}_\alpha(\tilde{x})\big)=\arg\min_{y_1,y_2} \mathbb{E}[L(\tilde{x},y_1,y_2)]$. $L$ is not unique here, with a general form given in \cite{fissler2016higher}. Let $q(s,a)$ denote the $\alpha$-VaR action-value function. Based on joint elicitability, \cite{coache2023conditionally} proposed to learn both $q(s,a)$ and $Q(s,a)$ simultaneously by minimizing $\mathbb{E}[L(\tilde{r}+\gamma V(\tilde{s}'),q(s,a), Q(s,a)]$. It is a challenging task due to the non-convexity of $L$ and the fact that the partial derivatives of $L$ involve compound terms that may hinder stability (see discussion in Appendix~\ref{app:joint-elic-cvar}). To address this, we instead propose a two-time-scale learning scheme. Namely, while VaR is elicitable via quantile regression, CVaR can be expressed as an expectation given VaR (Eq.~\ref{eq:def-cvar}), which leads to
\begin{equation*}
\begin{aligned}
    q(s,a)&=\arg\min_{y_1}\mathbb{E}[l_\alpha(\tilde{r}+\gamma V(\tilde{s}')-y_1)],~l_\alpha(x-y):=(\alpha-\mathbb{I}\{x< y\})(x-y)\\
    Q(s,a)&=\arg\min_{y_2}\frac12\big(y_2 - q(s,a)+\frac{1}{\alpha}\mathbb{E}[(q(s,a)-\tilde{r}-\gamma V(\tilde{s}'))_+]\big)^2,
\end{aligned}
\end{equation*}
where $V(s')=\mathbb{E}_{a'\sim\pi(\cdot|s')}[Q(s',a')]$. Similar to Eq.~\ref{eq:expectile-qlearning}, this formulation leads to stochastic gradient-based Q-learning-style updates for both $q$ and $Q$. Given a transition $(s,a,r,s')$ and learning rates $\zeta^q$ and $\zeta^Q$ (with $\zeta^q\gg \zeta^Q$, thus $Q$ is quasi-static w.r.t. $q$), the updates are
\begin{equation}
\label{eq:cvar-qlearning}
\begin{aligned}
q_{t+1}(s,a)\leftarrow& q_t(s,a) - \zeta^q_t(s,a) \cdot (\mathbb{I}\{r+\gamma V_t(s')< q_t(s,a)\}-\alpha)\\
Q_{t+1}(s,a)\leftarrow& Q_t(s,a) - \zeta^Q_t(s,a) \cdot (Q_t(s,a)-q_t(s,a)+\frac{1}{\alpha}(q_t(s,a)-r-\gamma V_t(s'))_+)
\end{aligned}
\end{equation}
The following propositions show that the value learning rules above enjoy convergence guarantees that are comparable to those in standard
Q-learning. Proposition~\ref{prop:expectile} follows the Robbins-Monro conditions as for risk-neutral Q-learning~\citep{bertsekas1996neuro}, while Proposition~\ref{prop:cvar} depends on the two-time scale framework in \cite{borkar2023stochastic}.
\begin{proposition}
\label{prop:expectile}
    With assumption~\ref{ass:1}, for Eq.~\ref{eq:expectile-qlearning}, if the nonnegative learning rates $\zeta_t(s,a)$ satisfy $\sum_{t=0}^{\infty} \zeta_t(s,a)=\infty$ and $\sum_{t=0}^\infty \zeta^2_t(s,a) <\infty$ for all $(s,a)\in\mathcal{S}\times\mathcal{A}$, then $Q_t(s,a)$ converges to $Q^\pi(s,a)$ associated with dynamic expectile for all $(s,a)$ almost surely.
\end{proposition}
\begin{proposition}
\label{prop:cvar}
    With assumption~\ref{ass:1} and assuming that the sequences of $q_t$ and $Q_t$ are bounded almost surely, for Eq.~\ref{eq:cvar-qlearning}, if the nonnegative learning rates $\zeta^q_t(s,a),\zeta^Q_t(s,a)$ satisfy $\sum_{t=0}^{\infty} \zeta^q_t(s,a)=\sum_{t=0}^\infty\zeta^Q_t(s,a)=\infty$, $\sum_{t=0}^\infty [\zeta^q_t(s,a)]^2<\infty$, $ \sum_{t=0}^\infty [\zeta^Q_t(s,a)]^2<\infty$, and $\lim_{t\rightarrow\infty}\frac{\zeta^Q_t(s,a)}{\zeta^q_t(s,a)}=0$ for all $(s,a)\in\mathcal{S}\times\mathcal{A}$, then $Q_t(s,a)$ converges to $Q^\pi(s,a)$ associated with dynamic CVaR for all $(s,a)$ almost surely.
\end{proposition}

\subsection{Actor-Critic Algorithm}
It is now convenient to construct actor-critic algorithm with stochastic policies for dynamic expectile and CVaR based on Sec.~\ref{sec:learn-policy} and Sec.~\ref{sec:learn-value}. For policy updates, a common way to estimate the gradient in Eq.~\ref{eq:surrogate-pg} is via Monte Carlo sampling, i.e., trajectories $\{s_t,a_t,r_t\}$ are generated by executing $\pi(\cdot|s)$ in the environment, and the collected samples are subsequently used to update both policy and value function. These trajectories are discarded after each policy update (if no importance sampling is applied). This is called on-policy policy gradient (PG) as in e.g., REINFORCE with baseline~\citep{sutton1998reinforcement} and PPO~\citep{schulman2017proximal}. In essence, on-policy PG estimates the expectation over actions in Eq.~\ref{eq:surrogate-pg} (since $\nabla_\theta \sum_a\pi_\theta(a|s)Q^\pi(s,a) =\mathbb{E}_{a\sim\pi_\theta(\cdot|s)}[Q^\pi(s,a)\nabla_\theta\log\pi_\theta(a|s)]$) by sampling a single action from $\pi(\cdot|s)$. In contrast, Expected Policy Gradient (EPG)~\citep{ciosek2020expected} exploits the structure of discrete action spaces to compute $\nabla_\theta \sum_a \pi_\theta(a|s)Q^\pi(s,a)$ exactly by enumerating all $a\in\mathcal{A}$ at a state $s$. By explicitly integrating over the entire action distribution, the resulting gradient no longer depends on the particular action executed in the environment. As a result, EPG enables policy updates using states drawn from arbitrary distributions, such as those in an experience replay buffer (i.e., off-policy), effectively decoupling policy learning from the behavior distribution. For value function learning, the backup value $r+\gamma V(s')$ in Eq.~\ref{eq:expectile-qlearning} and Eq.~\ref{eq:cvar-qlearning} can be computed, similar to EPG, by integrating over all actions, i.e. $r + \gamma \sum_{a'}\pi(a'|s')Q(s',a')$. This is also consistent with the idea of Expected SARSA~\citep{sutton1998reinforcement}. 

\begin{wrapfigure}{r}{0.475\textwidth}
\vspace{-0.2in}
\begin{algorithm}[H] 
    \caption{Actor-Critic for dynamic Expectile (and CVaR)}
    \label{algo:actor-critic} 
    
    \textbf{Input:} Risk level $\alpha$, training steps $M$, mini batch size $N$, $\pi$ learning rate $\eta_\theta$, $Q$ learning rate $\zeta_\phi$, {\color{brown}($q$ learning rate $\zeta_\psi$)}, soft update parameter $\tau$\;
    
    \textbf{Initialize:} Policy $\pi_\theta$, target policy $\pi_{\bar{\theta}}$, value function $Q_\phi$, target value function $Q_{\bar\phi}$, {\color{brown}(quantile value function $q_\psi$)}, buffer $B$\;

    $s=$ env.reset()\;
    \For{$m = 1, \dots, M$}{
        $a\sim\pi_\theta(\cdot|s)$; $r,s'\leftarrow$env.step(a) \;
         $B.\mathrm{add}(s,a,r,s')$\;
        \If{$B$.size()$\geq N$}{
        Sample $\{(s,a,r,s')_i\}_{i=1}^N\sim B$\;
        $V(s')=\mathbb{E}_{a'\sim\pi_{\bar\theta}(\cdot|s')}[Q_{\bar\phi}(s',a')]$ \;
        Update $Q_\phi$ {\color{brown}(and $q_\psi$)} via Eq.~\ref{eq:expectile-qlearning} or Eq.~\ref{eq:cvar-qlearning}\;
        Update $\pi_\theta$ via Eq.~\ref{eq:surrogate-pg} with $Q_\phi$ using EPG\;
        $\bar\theta\leftarrow \tau \theta + (1-\tau)\bar\theta$\;
        $\bar\phi\leftarrow \tau \phi + (1-\tau)\bar\phi$\;
        }
        $s\leftarrow s'$\;
        \If{env.done()}{
            $s=$ env.reset()\;
        }
    }
\end{algorithm}
\vspace{-0.2in}
\end{wrapfigure}
We adopt the ideas of EPG and Expected SARSA to form our algorithm. Note that EPG can be applied in both on-policy and off-policy settings. Here we focus on off-policy actor-critic for improved sample efficiency. It is well known that target networks are essential in off-policy time-difference learning to mitigate non-stationarity and prevent divergence of the value function~\citep{mnih2015human}. In our case, the backup value $r + \gamma \sum_{a'}\pi(a'|s')Q(s',a')$ depends on both $\pi$ and $Q$, and thus requires target networks for both of them. Specifically, for the policy $\pi_\theta(\cdot|s)$ and value function $Q_\phi(s,a)$, we introduce corresponding target networks $\pi_{\bar{\theta}}(\cdot|s)$ and $Q_{\bar{\phi}}(s,a)$, where $\bar{\theta}$ and $\bar{\phi}$ are updated via slow tracking of $\theta$ and $\phi$. In addition, when $\rho$ is CVaR, we introduce an auxiliary function $q_{\psi}(s,a)$ to estimate quantile. Given a transition $(s,a,r,s')$ sampled from the replay buffer, the policy $\pi_\theta(\cdot|s)$ is updated using EPG as in Eq.~\ref{eq:surrogate-pg}, while $Q_\phi$ (and $q_\psi$ if applicable) is updated according to Eq.~\ref{eq:expectile-qlearning} or Eq.~\ref{eq:cvar-qlearning}, where $V(s')$ is calculated using the target networks as $\sum_a \pi_{\bar{\theta}}(a|s)Q_{\bar{\phi}}(s,a)$. For continuous action spaces, the expectation in both policy gradient and value learning can be approximated in practice by sampling a sufficient number of actions. The full algorithm is in Algo.~\ref{algo:actor-critic}, with additional operations for dynamic CVaR marked in parenthesis.

\subsection{Related Work}
\label{sec:related-work}
The works most closely related to ours, i.e., using stochastic policies for dynamic expectile and CVaR, are \cite{jiang2024learning} and \cite{coache2023conditionally}. 

\cite{jiang2024learning} proposed a mode-free, on-policy PG approach for dynamic expectile based on elicitability. Although the authors do not explicitly state that their algorithm optimizes dynamic expectile, we show that it indeed does, see Appendix~\ref{app:aaai2024}. Analogous to Eq.~\ref{eq:expectile-qlearning}, their Bellman operator for dynamic expectile is given by $\mathcal{T}^\pi V(s)=V(s) + 2\zeta\mathbb{E}_{a\sim \pi(\cdot|s)}\mathbb{E}_{r,s'}\big[(1-\alpha)\big(\delta(s,a,r,s')\big)_{-}+\alpha\big(\delta(s,a,r,s')\big)_+ \big]$, where $\delta(s,a,r,s') = r(s,a) + \gamma V(s') - V(s)$, and $\zeta$ is the step-size. Defining the advantage function as $A^\pi(s,a):=2\zeta \mathbb{E}_{r,s'}\big[(1-\alpha)\big(\delta(s,a,r,s')\big)_{-}+\alpha\big(\delta(s,a,r,s')\big)_+ \big]$, the policy is updated using the standard PG form $ \mathbb{E}_s[A^\pi(s,a)\nabla_\theta\log\pi_\theta(a|s)]$. Their algorithm is instantiated within the PPO framework~\citep{schulman2017proximal}.

\cite{coache2023conditionally} developed an on-policy actor–critic method for dynamic CVaR. As discussed in Sec.~\ref{sec:learn-value}, they learn the value function in a model-free manner by minimizing a jointly elicitable loss for the (VaR, CVaR) pair. The policy gradient, however, still requires solving for the optimal risk envelope via a Lagrangian formulation, i.e., $\max_\pi V(s)=\max_\pi\mathrm{CVaR}_\alpha(\tilde{r}+\gamma V(\tilde{s}'))=\max_\pi \min_{\xi}\max_{\lambda}\mathbb{E}[(\tilde{r}+\gamma V(\tilde{s}'))\xi(\tilde{r},\tilde{s}')]+\lambda (1-\mathbb{E}[\xi(\tilde{r},\tilde{s}')])$. For CVaR, the optimal dual variable corresponds to the VaR (cf. Eq.~\ref{eq:def-cvar}), thus they replaced $\lambda$ with the learned VaR function, denoted as $q(s)$. This yields an estimate of the risk envelope $\xi(r,s'|s,a)=\frac{1}{\alpha}\mathbb{I}\{r(s,a)+\gamma V(s') < q(s)\}$, and the resulting PG takes the form $\mathbb{E}_s\big[\frac{1}{\alpha}\big(r+\gamma V(s')-q(s)\big)_{-}\nabla_\theta \log\pi_\theta(a|s)\big]$ (see Eq.~6.16b and Eq.~L2 in \cite{coache2023conditionally}). Note that our formulation differs slightly, as we consider left-tail CVaR, whereas \cite{coache2023conditionally} focus on the right tail.

\section{Experiments}
\label{sec:exp}
Following \cite{luo2024simple} and \cite{mead2025return}, we modify several domains such that the risk-averse behavior is clear to identify to evaluate the algorithms. We start from two tabular cases: a maze domain adapted from \cite{greenberg2022efficient} and the Cliffwalk environment from \cite{sutton1998reinforcement}, then followed with LunarLander from the Box2D environments of OpenAI Gym~\citep{brockman2016openai}, and InvertedPendulum from the Mujoco~\citep{todorov2012mujoco} environments of OpenAI Gym. In the tabular domains, all algorithms are implemented in a tabular manner, whereas in the remaining environments, function approximation with neural networks is used. Since we are interested in learning risk-averse policies, we choose $\alpha<0.5$ for expectile and $\alpha<1$ for CVaR. Additional implementation and parameter details are provided in Appendix~\ref{app:exp-details}.

\begin{figure}[t]
    \begin{center}
        \includegraphics[width=0.84\textwidth]{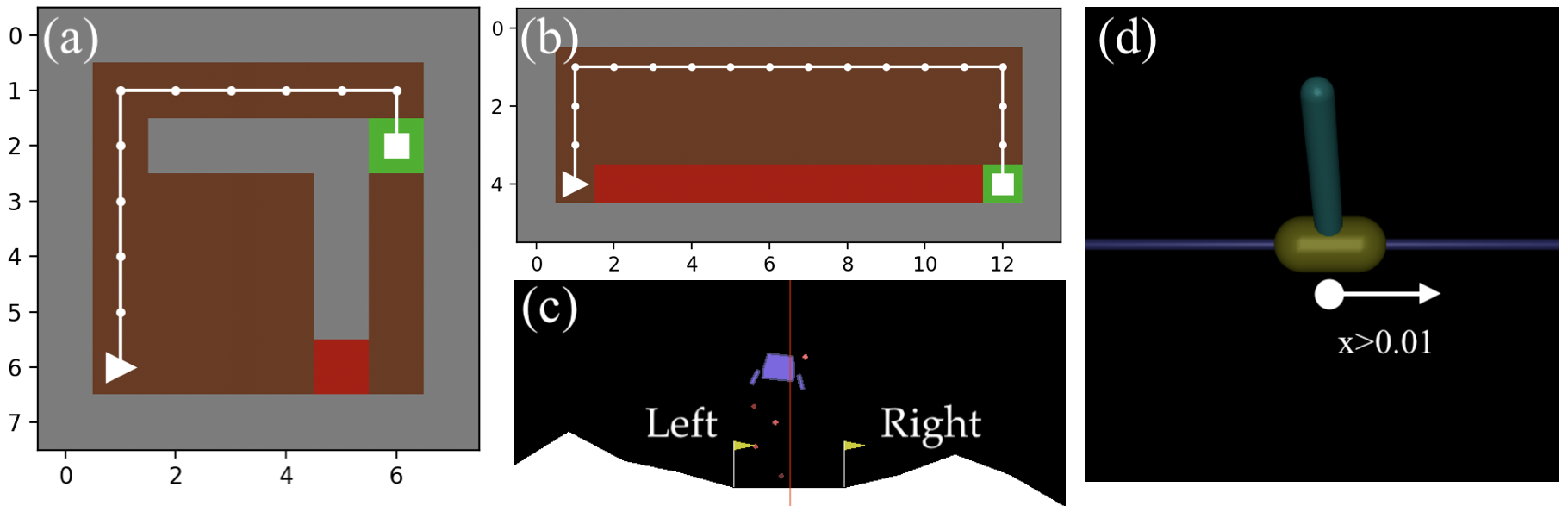}
    \end{center}
    \caption{(a) Maze. Visiting red state receives a random reward, with mean $-1$. (b) Cliffwalk. Some states in row index 2 and 3 have chances to slide towards the cliff, see text. (c) LunarLander. Visiting all states in the right part receives a random reward with mean $0$ per step. (d) Inverted Pendulum. Staying in the $x>0.01$ region receives a random reward with mean $0$ per step.}
    \label{fig:envs}
    \vspace{-0.14in}
\end{figure}

\textbf{Baselines.} We compare our method with methods in  \cite{jiang2024learning} (denoted as Exp-PPO) and \cite{coache2023conditionally}. However, the method in \cite{coache2023conditionally} generally does not work well, we report its learning curves separately in Fig.~\ref{fig:coache-res} in appendix. We also include Q-learning~\citep{sutton1998reinforcement} (denoted as QL) and EPG~\citep{ciosek2020expected} as risk-neutral baselines. We denote the actor-critic for dynamic expectile and CVaR proposed in this paper as Exp-AC and CVaR-AC.

\subsection{Maze and Cliffwalk}
The two domains are shown in Fig.~\ref{fig:envs} (a) and (b). For both of them, the gray color in the figures marks the walls. The agent starts from the white triangle state and aims to reach the green goal state. The action space is discrete with four actions \{Up, Down, Left, Right\}. For Maze, the per-step reward is $-1$ before reaching the goal except the red state. Visiting the red state receives a random reward $-1+\mathcal{N}(0,1)\times 30$, then clipped to the range $[-20,20]$. The reward for visiting the goal is $10$. Thus, the longer path (marked in white color) is risk-averse while the shortest path going through the red state is risk-neutral. For Cliffwalk, the reward of every safe step is $-1$. The red states are cliff and visiting them results in a $-100$ reward and ends the game. For all states right above the cliff (i.e, row index is $3$ and column index in $[2,11]$), every left or right movement has $20\%$ chance to slide to the cliff. For states whose row index is $2$ and column index in $[2,7]$, every left or right movement has $10\%$ chance to slide to row $3$. This design ensures that the path going through the middle lane (row index is 2) is risk-neutral, while the longer path (marked in white color) is risk-averse. We set $\alpha=0.05$ for Exp-AC and Exp-PPO, and $\alpha=0.1$ for CVaR-AC.

Agents are evaluated by sampling $10$ trajectories during training and we report the risk-averse path rate of sampled trajectories, as shown in Fig.~\ref{fig:sub1} and \ref{fig:sub2}. In both domains, the three risk-averse algorithms can successfully identify the risk-averse path, with Exp-AC and CVaR-AC learning faster than Exp-PPO, since off-policy learning is more sample efficient than on-policy learning.

\subsection{LunarLander}
This domain is illustrated in Fig.~\ref{fig:envs} (c). The objective is to control the lander’s engine to achieve a safe landing. We consider the discrete-action version of this environment and refer readers to the official documentation for a detailed description. As shown in the figure, the environment is partitioned into left and right regions by the center line of the landing pad. In \cite{luo2024simple} and \cite{mead2025return}, a random reward is introduced only when the agent lands on the right part of the ground. In contrast, since we consider dynamic risk, we inject noise at each time step rather than only at termination. Specifically, we modify the environment such that an additional noisy reward $\mathcal{N}(0,1)\times 3$, clipped to $[-6,6]$, is added whenever the agent remains in the right region, including both during flight and upon landing. Therefore, a risk-averse agent is expected to avoid entering the right region while still achieving a safe landing. The maximum episode length is 500. $\alpha=0.3\in\{0.2,0.3,0.4\}$ for Exp-AC. $\alpha=0.9\in\{0.5,0.7,0.9\}$ for CVaR-AC. $\alpha=0.4\in\{0.2,0.3, 0.4\}$ for Exp-PPO.

The return and the rate of staying in non-noisy (left) region in an episode are shown in Fig.~\ref{fig:sub3} and \ref{fig:sub4}. Since dynamic risk is difficult to measure directly in practice, we additionally report the empirical 0.2-CVaR of the return in Fig.~\ref{fig:sub5} as a reference. This is estimated by sampling 15 episodes during evaluation and averaging the three worst returns. Exp-AC achieves both a higher risk-averse rate and a higher 0.2-CVaR compared to the risk-neutral EPG baseline. It also achieves a high left landing rate as shown in Fig.~\ref{fig:ll-left} in appendix. CVaR-AC is more challenging to tune in this domain and exhibits slightly lower performance than EPG, although it still demonstrates a clear risk-averse tendency (cf. Fig.~\ref{fig:sub4}). In contrast, Exp-PPO fails to learn a reasonable policy in this setting.

\subsection{Inverted Pendulum}
This domain is illustrated in Fig.~\ref{fig:envs} (d). The objective is to keep the inverted pendulum upright (within a certain angle limit) as long as possible. The original per-step reward is $1$. We add a noise $\mathcal{N}(0,1)\times 10$, clipped to $[-20,20]$, to the per-step reward if $X$-position$>0.01$. Thus, a risk-averse agent is expected to stay away from the noisy reward region while keep the balance of the pendulum. The maximum episode length is 300. $\alpha=0.4\in\{0.2,0.3,0.4\}$ for Exp-AC. $\alpha=0.9\in\{0.5,0.7,0.9\}$ for CVaR-AC. $\alpha=0.4\in\{0.2,0.3, 0.4\}$ for Exp-PPO.

The return, $X<0.01$ rate in an episode, and 0.2-CVaR of return are presented in Fig.~\ref{fig:sub6}, \ref{fig:sub7}, and \ref{fig:sub8}. Both Exp-AC and CVaR-AC learn risk-averse policies by holding the pendulum more in the non-noisy region compared with EPG. Exp-PPO fails to learn a reasonable policy within the same environment step. Curves with more steps for Exp-PPO are shown in Fig.~\ref{fig:exp-ppo-res} in appendix, which exhibits unstable policy learning.

\textbf{Discussion.} The method in \cite{coache2023conditionally} generally performs poorly across the considered domains. One possible reason is that its policy gradient relies on the optimal dual variable and transition perturbation, which are instead approximated from learned value functions, potentially introducing additional error. For Exp-PPO~\citep{jiang2024learning}, the advantage function is nonlinear (cf. Sec.~\ref{sec:related-work}), making it difficult to construct a generalized advantage estimation (GAE)~\citep{schulman2015high} as in the risk-neutral PPO. Although the authors propose computing the advantage via multi-step transitions by repeatedly applying their Bellman operator, we observe that this approach leads to degraded performance in our experiments. CVaR-AC requires learning one policy and two value functions, resulting in a three-timescale algorithm. This increases the difficulty of hyperparameter tuning, as observed in LunarLander. We also observe that in non-tabular domains, both Exp-AC and CVaR-AC suffer performance degradation when $\alpha$ is set too small. A possible explanation is the presence of epistemic uncertainty during learning. Overly strong risk-aversion may hinder exploration and slow down policy improvement.

\begin{figure}[H]
    \centering

    \begin{subfigure}[t]{0.3\textwidth}
        \centering
        \includegraphics[width=\linewidth]{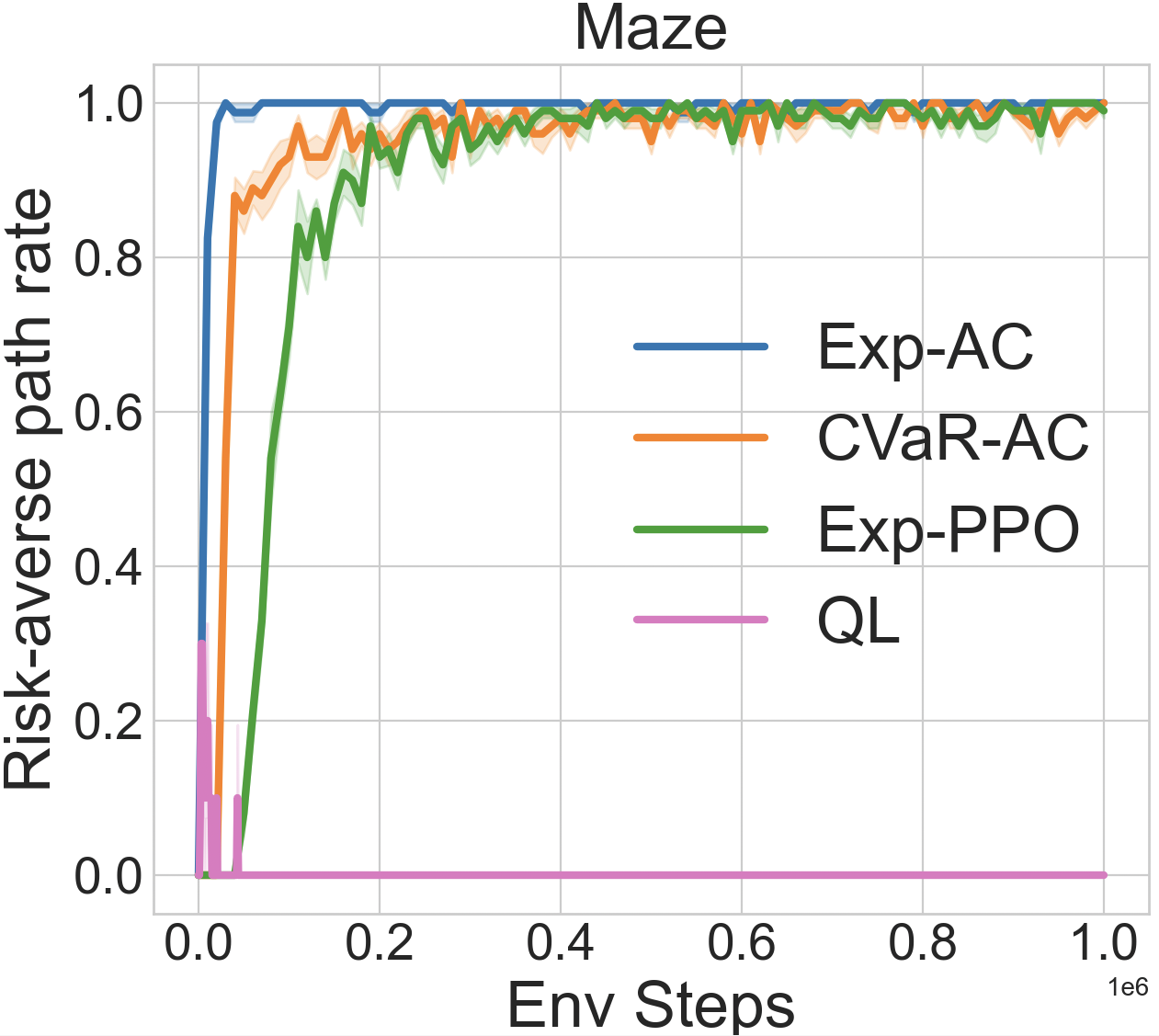}
        \caption{Risk-averse rate in maze}
        \label{fig:sub1}
    \end{subfigure}
    \hspace{0.05\textwidth}
    \begin{subfigure}[t]{0.3\textwidth}
        \centering
        \includegraphics[width=\linewidth]{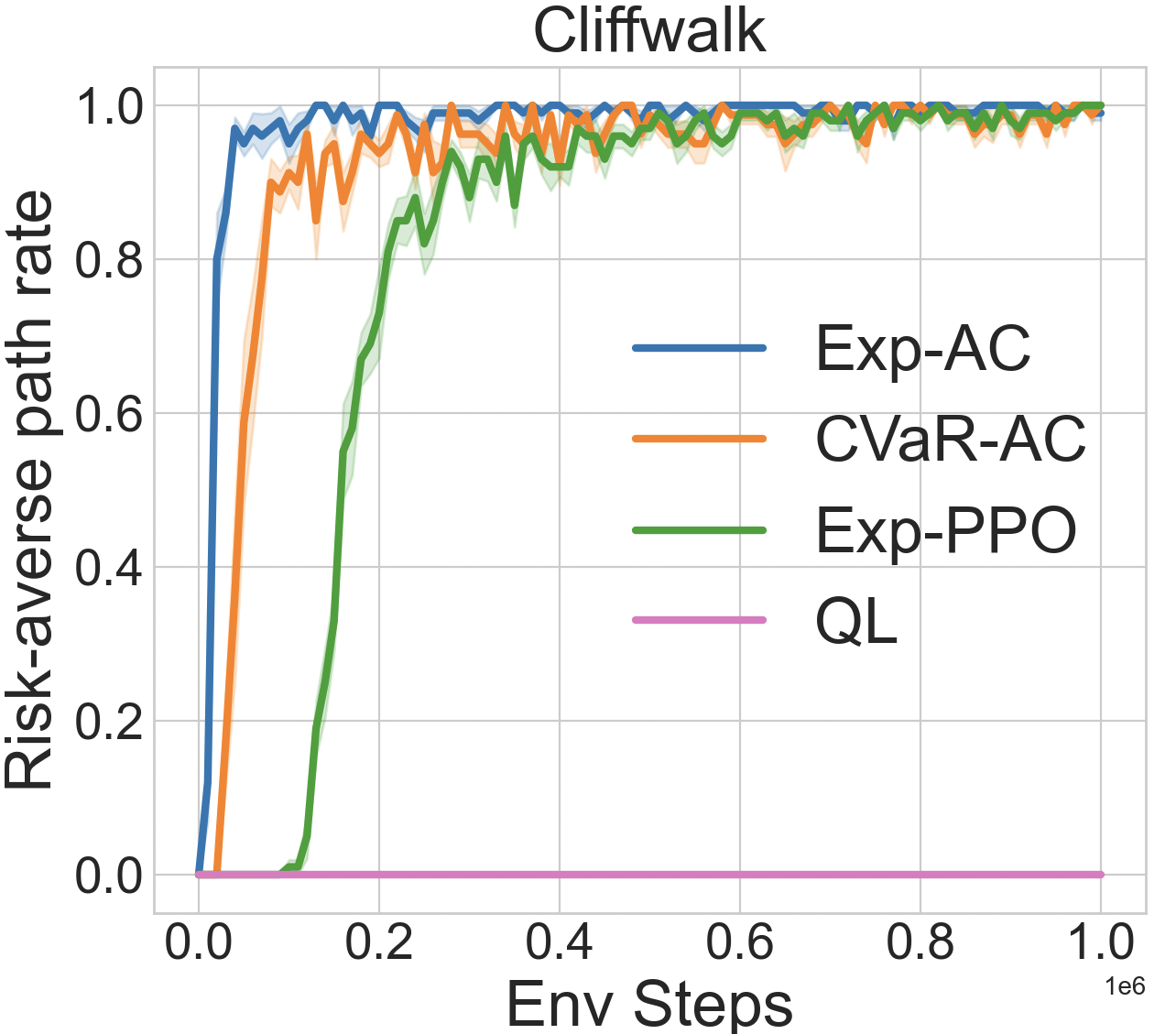}
        \caption{Risk-averse rate in cliffwalk}
        \label{fig:sub2}
    \end{subfigure}


    \begin{subfigure}[t]{0.31\textwidth}
        \centering
        \includegraphics[width=\linewidth]{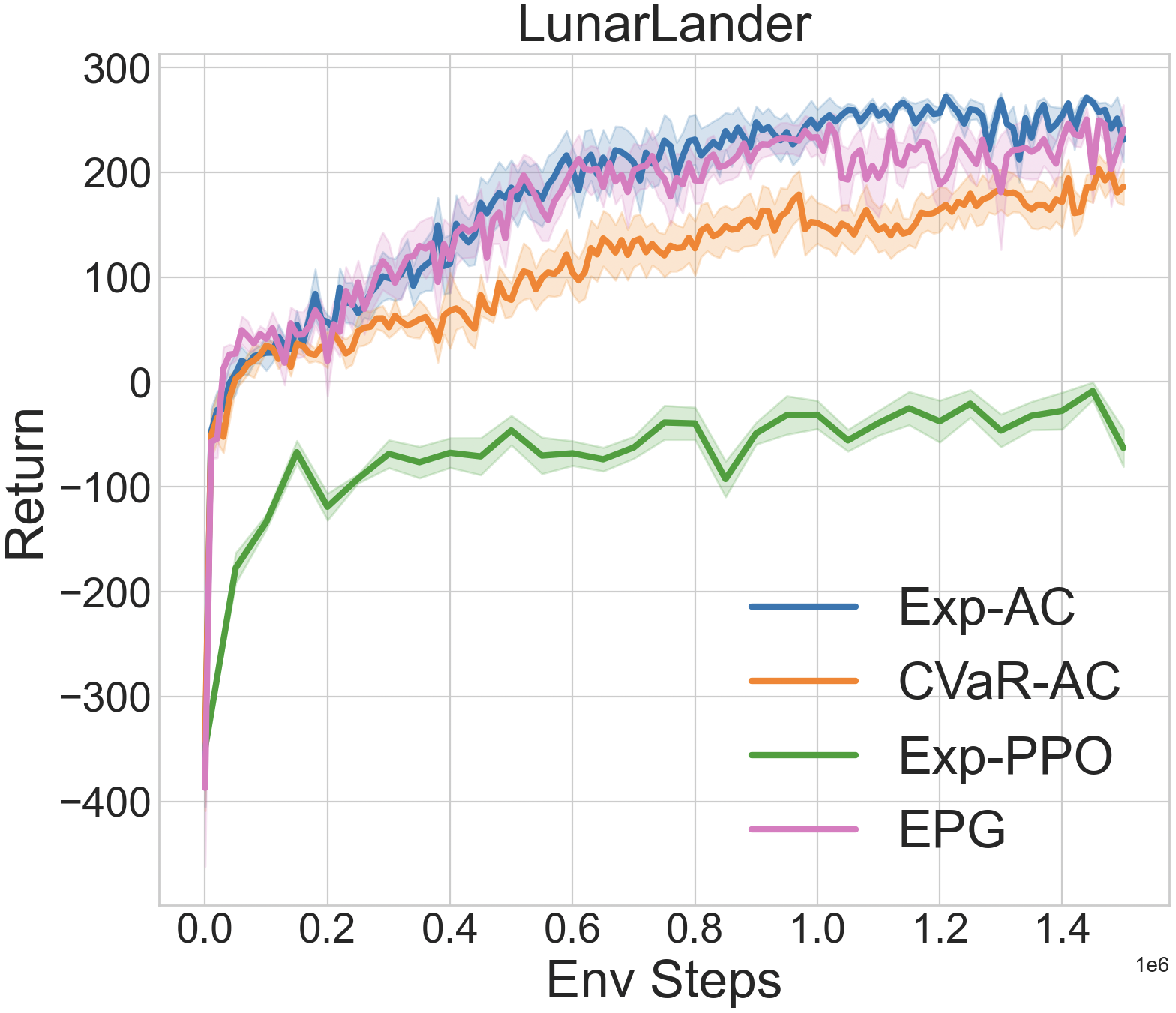}
        \caption{Expected return in Lunarlander}
        \label{fig:sub3}
    \end{subfigure}
    \hspace{0.02\textwidth}
    \begin{subfigure}[t]{0.31\textwidth}
        \centering
        \includegraphics[width=\linewidth]{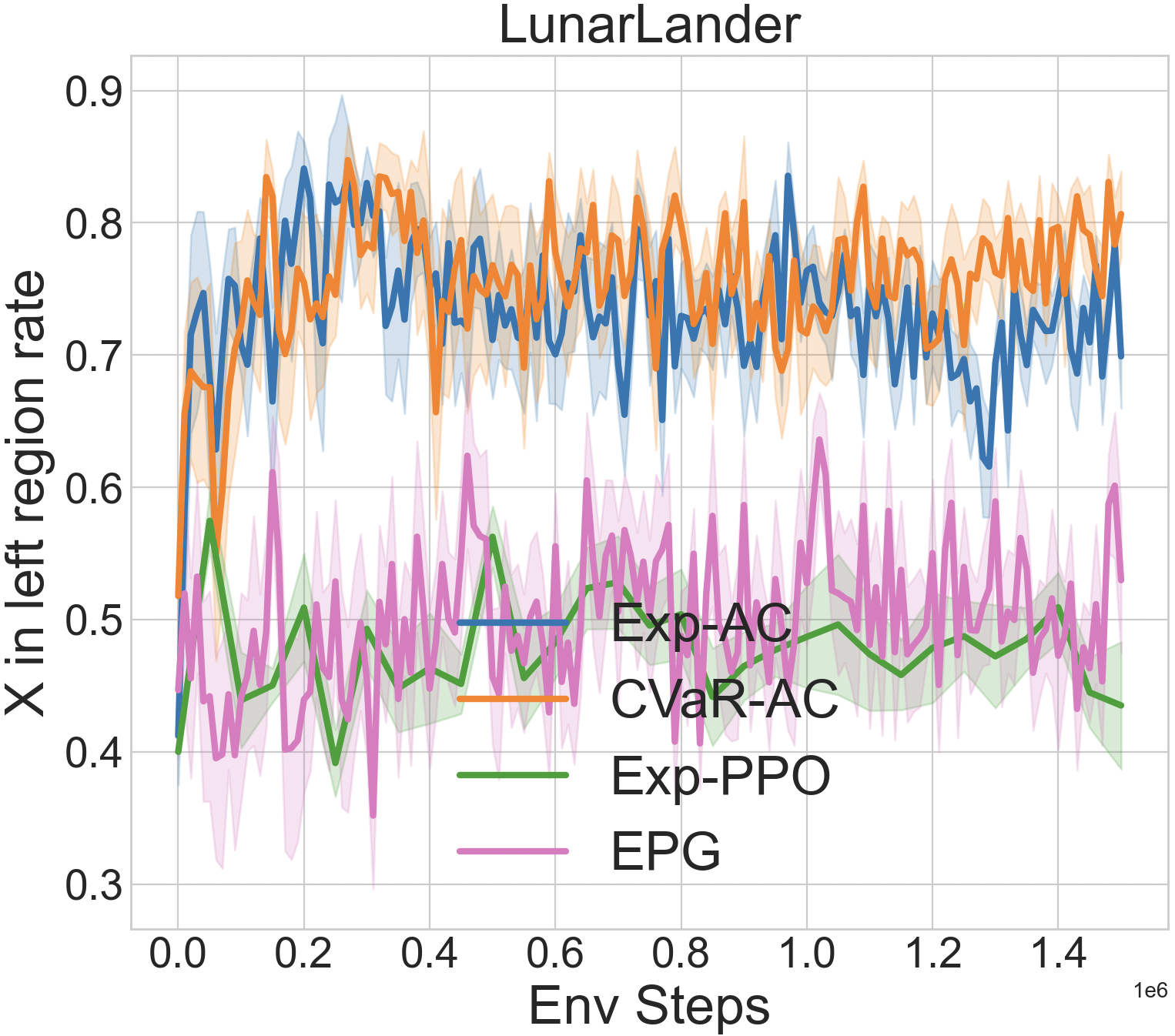}
        \caption{Risk-averse rate in Lunarlander}
        \label{fig:sub4}
    \end{subfigure}
    \hspace{0.02\textwidth}
    \begin{subfigure}[t]{0.31\textwidth}
        \centering
        \includegraphics[width=\linewidth]{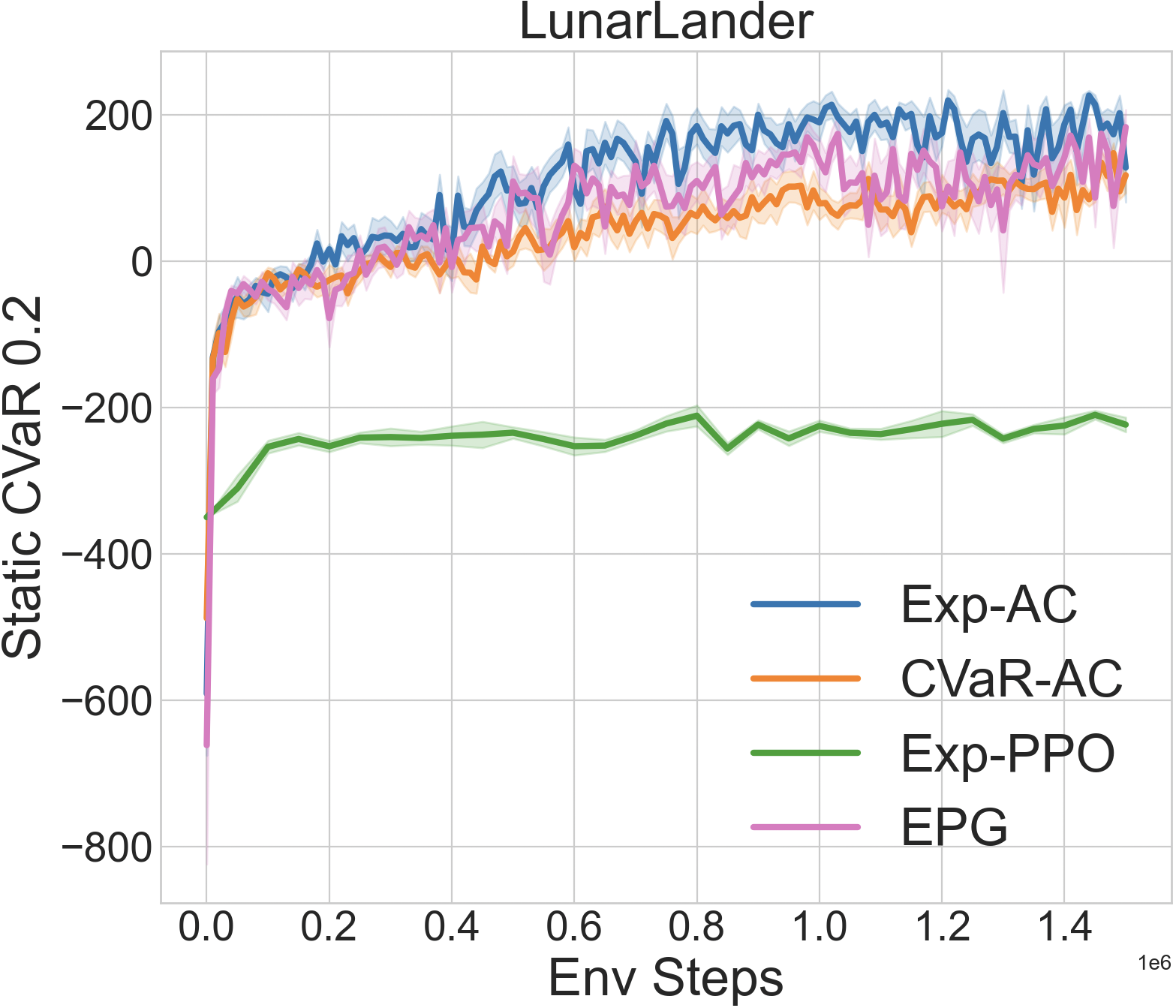}
        \caption{CVaR 0.2 of return in Lunarlander}
        \label{fig:sub5}
    \end{subfigure}

    \begin{subfigure}[t]{0.31\textwidth}
        \centering
        \includegraphics[width=\linewidth]{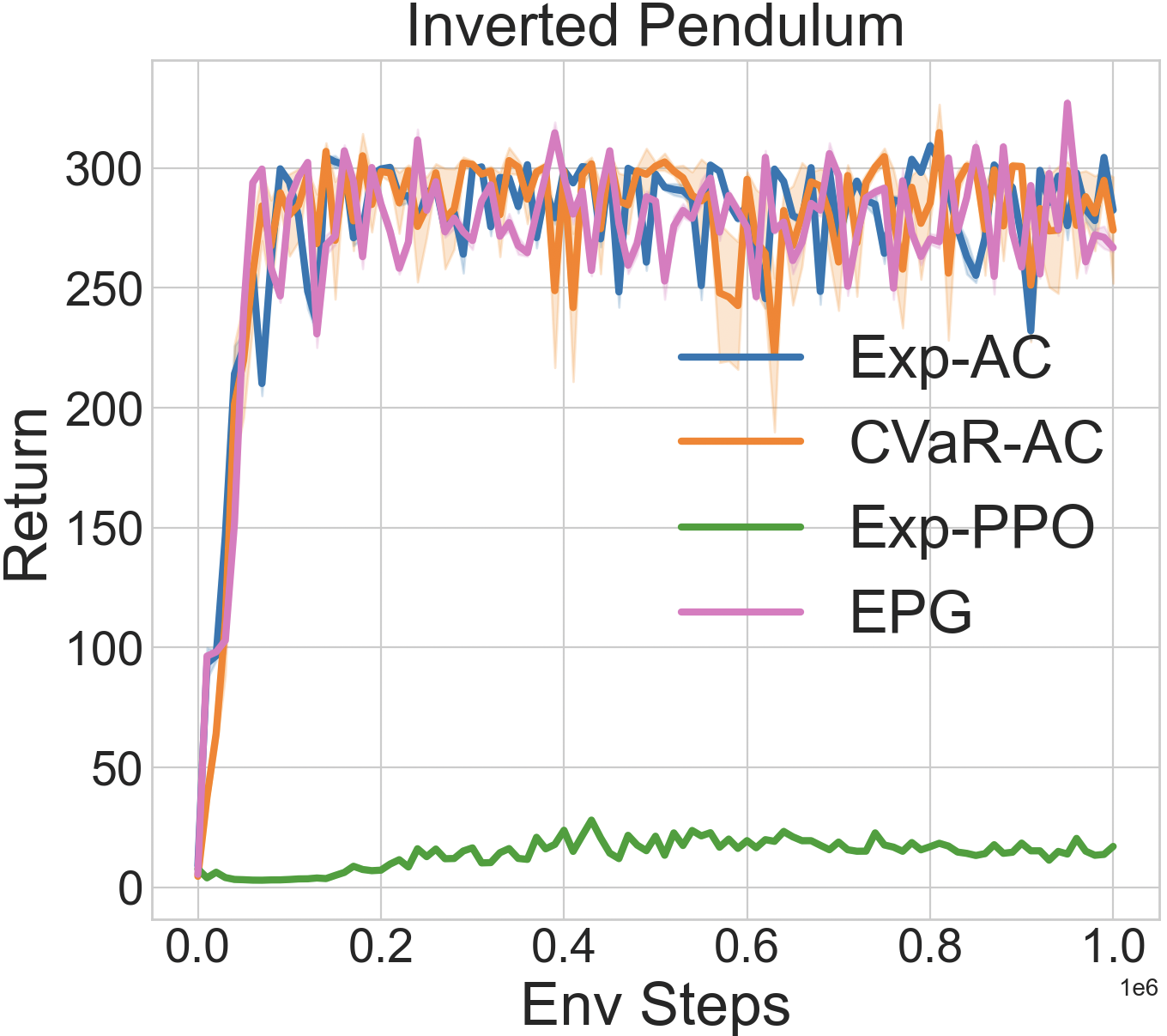}
        \caption{Expected return in Pendulum}
        \label{fig:sub6}
    \end{subfigure}
    \hspace{0.02\textwidth}
    \begin{subfigure}[t]{0.31\textwidth}
        \centering
        \includegraphics[width=\linewidth]{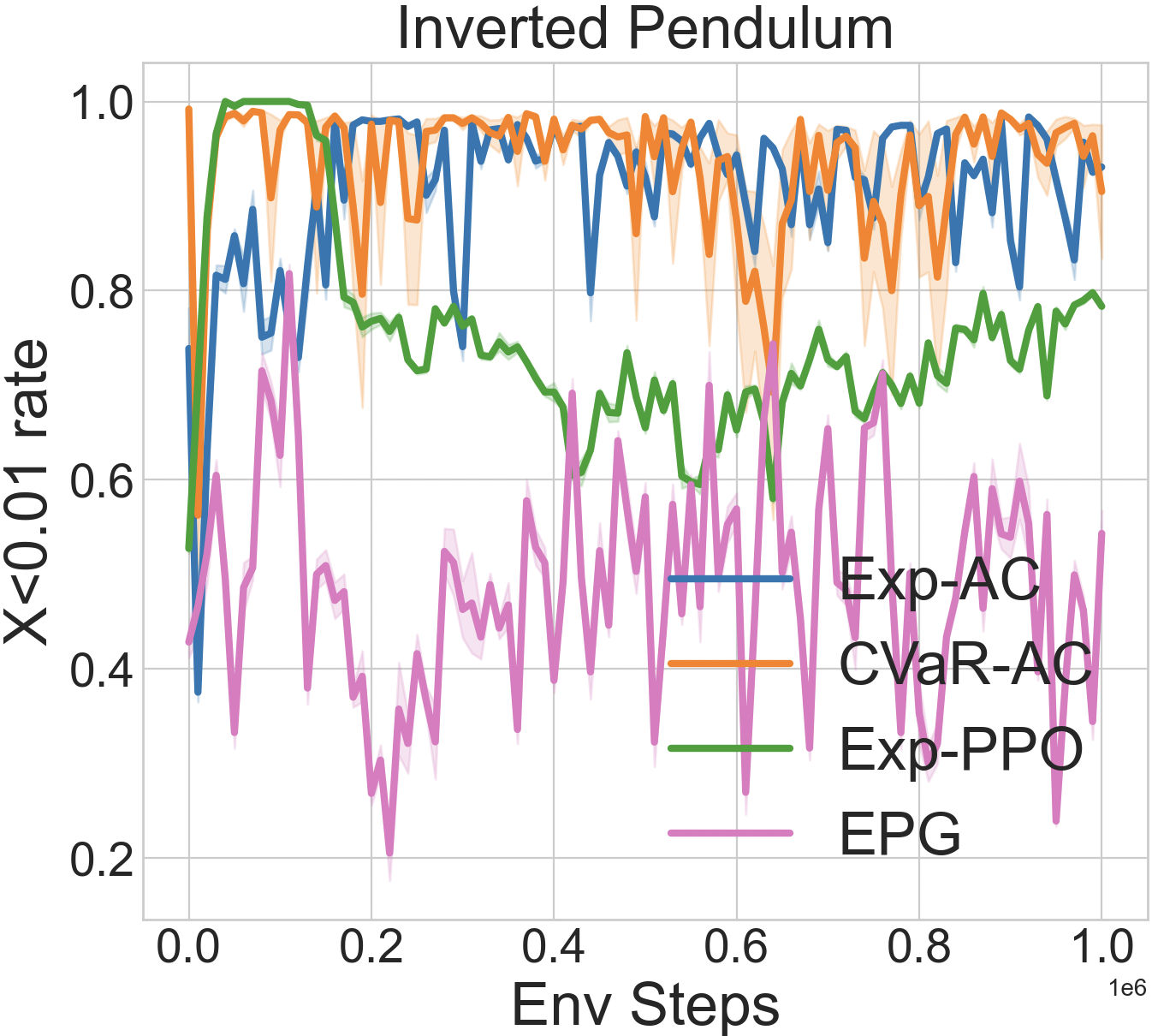}
        \caption{Risk-averse rate in Pendulum}
        \label{fig:sub7}
    \end{subfigure}
    \hspace{0.02\textwidth}
    \begin{subfigure}[t]{0.31\textwidth}
        \centering
        \includegraphics[width=\linewidth]{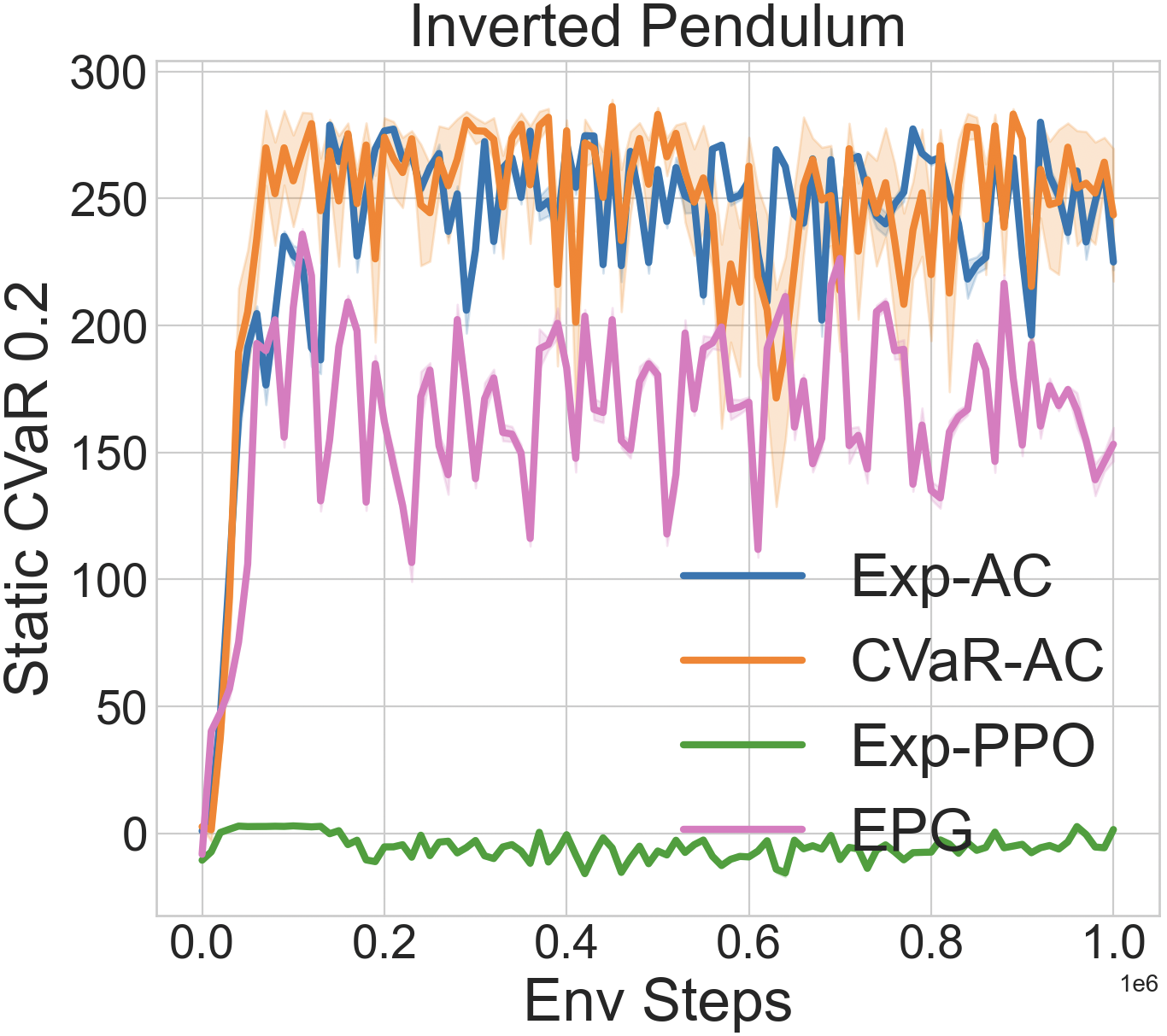}
        \caption{CVaR 0.2 of return in Pendulum}
        \label{fig:sub8}
    \end{subfigure}

    \caption{Risk-averse path rate in Maze and Cliffwalk. Expected return, risk-averse rate, and CVaR 0.2 of return in LunarLander and Inverted Pendulum. Curves are averaged over 10 seeds with shaded regions indicating standard errors.}
    \label{fig:ll-ivp-res}
\end{figure}
\section{Conclusion and Future Work}
\label{sec:conclusion}
This paper proposes a model-free off-policy actor–critic algorithm with stochastic policies for dynamic expectile and CVaR under softmax policy parameterization. The main advantages are that it avoids transition perturbation required by most stochastic-policy methods and improves sample efficiency via off-policy learning. We evaluate the method on both discrete and continuous action domains, under tabular settings and function approximation. Empirically, we show that our method succeeds in settings where existing approaches fail to learn a risk-averse or a reasonable policy.

\textbf{Limitations and future work.} The value learning methods in Sec.~\ref{sec:learn-value} are specific to expectile and CVaR and are not directly applicable to other risk measures. Moreover, optimizing dynamic CVaR leads to a three-timescale update scheme, which increases the difficulty of hyperparameter tuning. Observing the two limitations, extending the framework to other risk measures and improving the stability of dynamic CVaR training remain valuable for future work.


\section*{Acknowledgements}
The computing resources used in this work were provided by the Digital Research Alliance of Canada (alliancecan.ca). Yudong Luo is partially supported by a FRQNT postdoctoral scholarship and IVADO. Erick Delage is partially supported by the Canadian Natural Sciences and Engineering Research Council [Grant RGPIN-2022-05261].

\appendix
\section{Derivations and Proofs}
\subsection{SPG for Coherent Dynamic Risk}
\label{app:spg}
The objective we consider for dynamic risk with $\pi(\cdot|s)$ is
\begin{equation*}
    \rho_\infty^\pi(M)=\sum_{a_0} \pi(a_0|s_0)\rho\Big(\tilde{r}(s_0,a_0)+\gamma \sum_{a_1}\pi(a_1|\tilde{s}_1)\rho\big(\tilde{r}(\tilde{s}_1,a_1)+ \gamma \sum_{a_2}\pi(a_2|\tilde{s}_2)\rho(\tilde{r}(\tilde{s_2},a_2)+...)\big)\Big)
\end{equation*}
Similar to deterministic policy case, the action-value function and state value function are defined as $Q^\pi(s,a)=\rho^\pi_\infty(M|s_0=s,a_0=a)$ and $V^\pi(s)=\rho_\infty^\pi(M|s_0=s)$. It holds that $V^\pi(s)=\sum_a \pi(a|s)Q^\pi(s,a)$. The gradient of policy parameter $\theta$ w.r.t. $V^\pi(s_0)$ is computed as follows.
\begin{equation*}
\begin{aligned}
    \nabla_\theta V^{\pi_\theta}(s_0)&=\nabla_\theta \sum_{a_0}\pi_\theta(a_0|s_0)Q^{\pi_\theta}(s_0,a_0)\\
    &=\sum_{a_0}\nabla_\theta\pi_\theta(a_0|s_0)Q^{\pi_\theta}(s_0,a_0)+\pi_\theta(a_0|s_0)\nabla_\theta Q^{\pi_\theta}(s_0,a_0)\\
    &\overset{(a)}{=}\mathbb{E}_{a_0\sim\pi_\theta(\cdot|s_0)}\Big[Q^{\pi_\theta}(s_0,a_0)\nabla_\theta\log\pi_\theta(a_0|s_0)+\nabla_\theta Q^{\pi_\theta}(s_0,a_0)\Big]\\
    &\overset{(b)}{=}\mathbb{E}_{a_0\sim\pi_\theta(\cdot|s_0)}\Big[Q^{\pi_\theta}(s_0,a_0)\nabla_\theta\log\pi_\theta(a_0|s_0)+\nabla_\theta \rho(\tilde{r}_0+\gamma V^{\pi_\theta}(\tilde{s}_1))\Big]\\
    &\overset{(c)}{=}\mathbb{E}_{a_0\sim\pi_\theta(\cdot|s_0)}\Big[Q^{\pi_\theta}(s_0,a_0)\nabla_\theta\log\pi_\theta(a_0|s_0)+\nabla_\theta \mathbb{E}_{r_0,s_1\sim P^{\xi^*}(\cdot,\cdot|s_0,a_0)}\big[r_0+\gamma V^{\pi_\theta}(s_1)\big]\Big]\\
    &=\mathbb{E}_{a_0\sim\pi_\theta(\cdot|s_0)}\Big[Q^{\pi_\theta}(s_0,a_0)\nabla_\theta\log\pi_\theta(a_0|s_0)+ \gamma \mathbb{E}_{r_0,s_1\sim P^{\xi^*}(\cdot,\cdot|s_0,a_0)}\nabla_\theta V^{\pi_\theta}(s_1)\Big],
\end{aligned}
\end{equation*}
where (a) uses $\nabla_\theta \pi_\theta(a|s)=\pi_\theta(a|s)\nabla_\theta\log\pi(a|s)$; (b) follows the Bellman equation of $Q^\pi(s,a)$; (c) uses the dual representation of $\rho(\cdot)$ with the optimal risk envelope $P^{\xi^*}$.

Calculating $\nabla_\theta V^{\pi_\theta}(s_1)$ follows the same step as calculating $\nabla_\theta V^{\pi_\theta}(s_0)$. Putting all gradients back yields Eq.~\ref{eq:dynamic-risk-pg}.

\subsection{Surrogate SPG with Gaussian Policy}
\label{app:spg-gaussian}
Define $\pi_\theta(a|s)=\mathcal{N}\big(a;\mu_{\theta_1}(s),\Sigma_{\theta_2}(s)\big)$ with $\theta=\{\theta_1,\theta_2\}$. By the definition of Gaussian distribution
\begin{equation*}
\pi_\theta(a|s)=\frac{1}{\sqrt{(2\pi)^k|\Sigma_{\theta_2}(s)|}}\exp\Big(-\frac{1}{2}(a-\mu_{\theta_1}(s))^\top \Sigma_{\theta_2}(s)^{-1} (a-\mu_\theta(s))\Big),
\end{equation*}
where $k$ is the action dimension. Since we focus on taking gradient for $\theta_1$, we omit $\theta_2$ for $\Sigma$ in the derivation. The gradient of $\log\pi_\theta$ takes the form
\begin{equation*}
\begin{aligned}
    \log\pi_\theta(a|s)&=-\frac{k}{2}\log(2\pi)-\frac{1}{2}\log|\Sigma|-\frac{1}{2}(a-\mu_{\theta_1}(s))^\top \Sigma^{-1}(a-\mu_{\theta_1}(s)),\\
    \nabla_{\theta_1} \log\pi_\theta(a|s)&=\nabla_{\theta_1} \mu_{\theta_1}(s) \Sigma^{-1}(a-\mu_{\theta_1}(s)).
\end{aligned}
\end{equation*}
Putting back to the surrogate policy gradient
\begin{equation*}
\begin{aligned}
\nabla_{\theta_1}\int_{\mathcal{A}}\pi_\theta(a|s)Q^{\pi}(s,a) da&=\int_{\mathcal{A}}\pi_\theta(a|s)\nabla_{\theta_1}\log\pi_\theta(a|s)Q^\pi(s,a)ds\\
&= \nabla_{\theta_1}\mu_{\theta_1}(s) \Sigma^{-1}\int_{\mathcal{A}}\pi_\theta(a|s)(a-\mu_{\theta_1}(s))Q^\pi(s,a) da
\end{aligned}
\end{equation*}
By Stein's Lemma~\citep{stein1981estimation}, if $\tilde{x}$ follows $\mathcal{N}(\mu,\sigma^2)$ and $f$ is differentiable, then $\mathbb{E}[(\tilde{x}-\mu)f(\tilde{x})]=\sigma^2 \mathbb{E}[f'(\tilde{x})]$. Applying this to the $Q$ function, we have
\[
\int_{\mathcal{A}}\pi_\theta(a|s)(a-\mu_{\theta_1}(s))Q^\pi(s,a)=\Sigma \mathbb{E}_{a\sim\pi_\theta(\cdot|s)}[\nabla_a Q^\pi(s,a)].
\]
Therefore,
\begin{equation*}
\begin{aligned}
\nabla_{\theta_1}\int_{\mathcal{S}} d(s)\int_{\mathcal{A}}\pi_\theta(a|s)Q^\pi(s,a) dads&=\int_{\mathcal{S}}d(s)\nabla_{\theta_1}\mu_{\theta_1}(s)\Sigma^{-1}\Sigma \mathbb{E}_{a\sim\pi_\theta(\cdot|s)}[\nabla_aQ^\pi(s,a)]ds\\
&=\int_{\mathcal{S}}d(s)\nabla_{\theta_1}\mu_{\theta_1}(s) \mathbb{E}_{a\sim\pi_\theta(\cdot|s)}[\nabla_aQ^\pi(s,a)]ds
\end{aligned}
\end{equation*}
As a result, if $\Sigma=0$, then $\mathbb{E}_{a}[\nabla_a Q^\pi(s,a)]=\nabla_a Q^\pi(s,a)|_{a=\mu_{\theta_1}(s)}$, and the above gradient recovers DPG in Sec.~\ref{sec:deter-policy-case}. This calculation is the same as in the risk-neutral setting.

\subsection{Surrogate SPG with Softmax Policy}
\label{app:spg-softmax}
When $\pi(\cdot|s)$ is a softmax policy, $\pi_\theta(a|s)=\exp(\theta(s,a))/\sum_b \exp(\theta(s,b))$. Therefore
\begin{equation*}
    \frac{\partial \log \pi_\theta(a|s)}{\partial \theta(s,i)} = \mathbb{I}\{a=i\} - \pi_\theta(i|s)
\end{equation*}
For each $s$, the gradient to compute in Eq.~\ref{eq:surrogate-pg} is $\nabla_\theta d(s)\sum_a\pi_\theta(a|s)Q^\pi(s,a)$. Consider the partial derivative for $\theta(s,i)$
\begin{equation*}
\begin{aligned}
    \frac{\partial \sum_a \pi_\theta(a|s)Q^\pi(s,a)}{\partial \theta(s,i)} &=\sum_a Q^\pi(s,a)\frac{\partial \pi_\theta(a|s)}{\partial\theta(s,i)}=\sum_a Q^\pi(s,a)\pi_\theta(a|s)\frac{\partial \log\pi_\theta(a|s)}{\partial \theta(s,i)}\\
    &=\sum_a Q^\pi(s,a)\pi_\theta(a|s)\big(\mathbb{I}\{a=i\}-\pi_\theta(i|s)\big)\\
    &=\sum_a Q^\pi(s,a)\pi_\theta(a|s)\mathbb{I}\{a=i\} - \sum_a Q^\pi(s,a)\pi_\theta(a|s)\pi_\theta(i|s)\\
    &=Q^\pi(s,i)\pi_\theta(i|s) - \pi_\theta(i|s) V^\pi(s)\\
    &= \pi_\theta(i|s)\big(Q^\pi(s,a)-V^\pi(s)\big)
\end{aligned}
\end{equation*}
Therefore, $\forall (s,a)\in \mathcal{S}\times \mathcal{A}$, $\theta(s,a)$ is updated by $\theta(s,a)\leftarrow \theta(s,a)+\beta d(s)\pi_\theta(a|s)(Q^\pi(s,a)-V^\pi(s))$. This calculation is the same as in the risk-neutral setting.

\subsection{Proof of Lemma~\ref{lemma:monotone}}
\label{app:monotone}
We first recall a Lemma from~\cite{agarwal2021theory}.
\begin{lemma}{[Lemma 52 of \cite{agarwal2021theory}]}
\label{lemma:smooth}
    Fix a state $s$. Let $\theta_s\in \mathbb{R}^{|\mathcal{A}|}$ be the column vector of parameters for state $s$. Let $\pi_\theta(\cdot|s)$ be the corresponding vector of action probabilities given by the softmax parameterization. For some fixed vector $c\in\mathbb{R}^{|\mathcal{A}|}$, define:
    \[
    F(\theta):=\sum_a \pi_\theta(a|s) c_a.
    \]
    Then $||\nabla_{\theta_s}F(\theta_s) - \nabla_{\theta_s} F(\theta'_s)||_2 \leq \beta ||\theta_s - \theta'_s||_2$ where $\beta=5 ||c||_\infty$.
\end{lemma}
To prove Lemma~\ref{lemma:monotone}, we first prove the following lemma.
\begin{lemma}
\label{lemma:local-update}
    With assumption~\ref{ass:1} (specifically, $\tilde{r}\in[0,1]$ and $d(s)\in(0,1]$). Let $\pi_\theta$ be a softmax policy. For update rule in Eq.~\ref{eq:surrogate-pg}, if $\eta\leq (1-\gamma)/5$, then $\mathbb{E}_{a\sim\pi_{t+1}(\cdot|s)}[Q^{\pi_t}(s,a)]\geq \mathbb{E}_{a\sim\pi_t(\cdot|s)}[Q^{\pi_t}(s,a)]$ for all $s\in\mathcal{S}$.
\end{lemma}
\begin{proof}
    Let $F^Q_s(\theta):=d(s)\sum_a \pi_\theta(a|s)Q(s,a)$. By Lemma~\ref{lemma:smooth}, we have
    \begin{equation*}
    \begin{aligned}
        \|\nabla_{\theta_s}F_{s}^Q(\theta)-\nabla_{\theta_s'}F_{s}^Q(\theta')\|_2 &= \|d(s)(\nabla_{\theta_s}\sum_a \pi_{\theta}(a|s) Q(s,a)- \nabla_{\theta_s'}\sum_a \pi_{\theta'}(a|s) Q(s,a))\|_2\\
    &=d(s)\|\nabla_{\theta_s}\sum_a \pi_{\theta}(a|s) Q(s,a)- \nabla_{\theta_s'}\sum_a \pi_{\theta'}(a|s) Q(s,a)\|_2\\
    &\overset{(a)}{\leq} \|\nabla_{\theta_s}\sum_a \pi_{\theta}(a|s) Q(s,a)- \nabla_{\theta_s'}\sum_a \pi_{\theta'}(a|s) Q(s,a)\|_2 \\
    &\overset{(b)}{\leq} 5\|Q\|_\infty\|\theta(s,\cdot)-\theta'(s,\cdot)\|_2 \\
    &\overset{(c)}= 5/(1-\gamma)\|\theta(s,\cdot)-\theta'(s,\cdot)\|_2,
    \end{aligned}
    \end{equation*}
    where (a) follows $d(s)\leq 1$; (b) follows Lemma~\ref{lemma:smooth}; (c) is due to the assumption $\tilde{r}\in[0,1]$, then $||Q||_\infty\leq \frac{1}{1-\gamma}$.

    Hence, $F_{s}^Q(\theta)$ is $\beta:=5/(1-\gamma)$ smooth with respect to $\theta_s$. We further observe that the policy update can be rewritten as: $\theta_{t+1}(s,\cdot) := \theta_t(s,\cdot)+\eta\nabla_{\theta_s} F_s^{Q^{\pi_t}}(\theta_t),\;\forall s$. Recall that for a $\beta$ smooth function, gradient ascent will increase the function value provided that $\eta\leq 1/\beta$~\citep{beck2017first}. So we get that if $\beta \leq (1-\gamma)/5$, then $F_s^{Q^{\pi_t}}(\theta_{t+1})\geq F_s^{Q^{\pi_t}}(\theta_{t})$, which means $\sum_a\pi_{t+1}(a|s)Q^{\pi_t}(s,a)\geq \sum_a \pi_t(a|s)Q^{\pi_t}(s,a), \forall s$.
\end{proof}

Next, we prove Lemma~\ref{lemma:monotone}.
\Lemmamonotone*
\begin{proof}
    From Lemma~\ref{lemma:local-update}, we have $\mathbb{E}_{a\sim \pi_{t+1}(\cdot|s)}[Q^{\pi_t}(s,a)] \geq \mathbb{E}_{a\sim \pi_{t}(\cdot|s)}[Q^{\pi_t}(s,a)], \forall s$. Recall the Bellman operator in Eq.~\ref{eq:bellman-op-st}
    \[
    \mathcal{T}^\pi Q(s,a)=\rho\big(\tilde{r}(s,a)+\gamma \mathbb{E}_{a'\sim\pi(\cdot|\tilde{s}')}[Q(\tilde{s}',a')]\big)
    \]
    By the monotonicity of coherent risk measure and Lemma~\ref{lemma:local-update}, we have
    \[
    \rho\big(r(s,a) + \gamma \mathbb{E}_{a'\sim \pi_{t+1}(\cdot|\tilde{s}')}[Q^{\pi_{t}}(\tilde{s}',a')]\big) \geq \rho\big(r(s,a) + \gamma \mathbb{E}_{a'\sim \pi_t(\cdot|\tilde{s}')}[Q^{\pi_t}(\tilde{s}',a')]\big),
    \]
    which means $\mathcal{T}^{\pi_{t+1}}Q^{\pi_t}\geq \mathcal{T}^{\pi_t}Q^{\pi_t}=Q^{\pi_t}$.

    Futhermore, by the monotonicity of coherent risk measure, we have 
    \[
    Q^{\pi_t}\leq \mathcal{T}^{\pi_{t+1}} Q^{\pi_t}\leq (\mathcal{T}^{\pi_{t+1}})^2 Q^{\pi_t}\leq ...\leq\lim_{k\rightarrow \infty}(\mathcal{T}^{\pi_{t+1}})^k Q^{\pi_t}=Q^{\pi_{t+1}}.
    \]
    The monotone increasing of $V$ is due to
\[
\begin{aligned}
&V^{\pi_{t+1}}(s)-V^{\pi_t}(s)\\
=&\sum_a \pi_{t+1}(a|s)[Q^{\pi_{t+1}}(s,a)-Q^{\pi_t}(s,a)] + \sum_a [\pi_{t+1}(a|s)-\pi_t(a|s)]Q^{\pi_t}(s,a)
\end{aligned}
\]
The first term is non-negative since we have shown $Q^{\pi_{t+1}}\geq Q^{\pi_t}$. The second term is also non-negative from Lemma~\ref{lemma:local-update}. Thus $V^{\pi_{t+1}} \geq V^{\pi_t}$.
\end{proof}

\subsection{Proof of Theorem~\ref{thm:optimal}}
\label{app:policy-converge}

Our proof follows similar lines as the proof of Theorem 10 in \cite{agarwal2021theory}. One can start with a lemma that confirms the convergence of the iterates.

\begin{lemma}[Lemma 42 in \cite{agarwal2021theory}]
For all states $s \in \mathcal{S}$ and actions $a \in \mathcal{A}$, the sequences
$\{V^{\pi_t}(s)\}_{t \ge 0}$ and $\{Q^{\pi_t}(s,a)\}_{t \ge 0}$ are non-decreasing and converge:
\[
V^{\pi_t}(s) \rightarrow V^{\infty}(s),
\qquad
Q^{\pi_t}(s,a) \rightarrow Q^{\infty}(s,a).
\]
Define the limiting advantage function $A^{\infty}(s,a)
= Q^{\infty}(s,a) - V^{\infty}(s)$,
and let $\Delta
:= \min_{(s,a):\, A^{\infty}(s,a) > 0}
A^{\infty}(s,a)$.
Then $\Delta > 0$. Moreover, there exists a finite iteration index $T_0$ such that
for all $t > T_0$, all $s \in \mathcal{S}$, and all $a \in \mathcal{A}$, $Q^{\pi_t}(s,a)
> Q^{\infty}(s,a) - \Delta/4$.
\end{lemma}

The existence of $Q^\infty$ and $V^\infty$ is ensured by the monotone increasing of the value function (Lemma~\ref{lemma:monotone}) and that value functions are bounded above by $1/(1-\gamma)$.

The proof relies on defining for each state $s$, a partition of the actions into three sets:
\[
\begin{aligned}
    I^0_s:=&\{a|Q^\infty(s,a)=V^\infty(s)\},\\
    I_s^+:=&\{a|Q^\infty(s,a)>V^\infty(s)\},\\
    I_s^-:=&\{a|Q^\infty(s,a)<V^\infty(s)\},
\end{aligned}
\]
and showing in Part 1 that $\pi_t(a|s)\rightarrow 0$ for $a\in I^+_s\cup I^-_s$ followed with $I^+_s = \emptyset$ in Part 2. With the latter property in hand, one can show that 
\begin{align*}
\rho\big(\tilde{r}(s,a)+\gamma\cdot\max_{a'} Q^\infty(\tilde{s}',a)\big) &= \rho\big(\tilde{r}(s,a)+\gamma\cdot\max_{a'\in I^0_s\cup I_s^-} Q^\infty(\tilde{s}',a)\big)\\
&\leq \rho\big(\tilde{r}(s,a)+\gamma\cdot\max_{a'\in I^0_s\cup I_s^-} V^\infty(\tilde{s}')\big)    \\
&=  \rho\big(\tilde{r}(s,a)+\gamma V^\infty(\tilde{s}')\big) = Q^\infty(s,a),
\end{align*}
where we first exploit the fact that $I_s^+$ is empty and where the final equality follows form the fact that $Q^{\pi_t}(s,a) = \rho\big(\tilde{r}(s,a)+\gamma V^{\pi_t}(\tilde{s}')\big)$ for all $t$, hence must be satisfied in the limit. This can be combined to 
\begin{align*}
\rho\big(\tilde{r}(s,a)+\gamma\cdot\max_{a'} Q^{\pi_t}(\tilde{s}',a)\big) &\geq \rho\big(\tilde{r}(s,a)+\gamma\cdot\sum_{a'} \pi_t(a'|\tilde{s}') Q^{\pi_t}(\tilde{s}',a')\big) \\
&= \rho\big(\tilde{r}(s,a)+\gamma V^{\pi_t}(\tilde{s}')\big)   = Q^{\pi_t}(s,a),
\end{align*}
and therefore in the limit to obtain that $Q^\infty$ satisfies Eq. \ref{eq:bellman-opt-eq} and thus both $Q^\infty=Q^*$ and (by construction) $V^\infty = V^*$. This completes the proof of Theorem~\ref{thm:optimal}.

\subsubsection{Part 1}
We first show that $\pi_t(a|s)\rightarrow 0$ for $a\in I^+_s\cup I^-_s$ as $t\rightarrow \infty$. The proof of this part differs from that in \cite{agarwal2021theory}. In particular, \cite{agarwal2021theory} establishes the vanishing of the policy gradient by exploiting the smoothness of the associated antiderivative (see Lemma 44 therein), whereas our surrogate policy gradient does not admit a corresponding antiderivative with such structure.

Define $J(\theta) :=\sum_s d(s)\sum_a \pi_\theta(a|s)Q^{\pi_\theta}(s,a)$. 

Define the local primitive function as $\mathcal{L}_t(\theta):=\sum_s d(s)\sum_a \pi_\theta(a|s)Q^{\pi_{\theta_t}}(s,a)$. As a result, the surrogate policy gradient (Eq.~\ref{eq:surrogate-pg}) can be expressed as $g_t(\theta_t) := \nabla_\theta \mathcal{L}_t(\theta)|_{\theta = \theta_t}$.

\begin{lemma}
\label{lemma:gradient-0}
    With the same conditions in Lemma~\ref{lemma:monotone}, $g_t(\theta_t)\rightarrow 0$ as $t\rightarrow \infty$.
\end{lemma}
\begin{proof}
From Lemma~\ref{lemma:monotone},
\[
\begin{aligned}
&J(\theta_{t+1}) - \mathcal{L}_t(\theta_{t+1})\\
=&\sum_s d(s)\sum_a \pi_{\theta_{t+1}}(a|s) Q^{\pi_{\theta_{t+1}}}(s,a) - \sum_s d(s)\sum_a \pi_{\theta_{t+1}}(a|s) Q^t(s,a)\\
=&\sum_s d(s)\sum_a \pi_{\theta_{t+1}}(a|s)[Q^{\pi_{\theta_{t+1}}}(s,a)-Q^t(s,a)] \geq 0
\end{aligned}
\]
Also, from Lemma~\ref{lemma:smooth} and Lemma~\ref{lemma:local-update}, we have $\mathcal{L}_t(\theta)$ is $\beta:= 5/(1-\gamma)$ smooth. For a $\beta$-smooth function $f$, it holds that (see Lemma 3.4 of \cite{bubeck2015convex})
\[
|f(\theta_{t+1}) - f(\theta_t)-\nabla_\theta f(\theta_t)^\top(\theta_{t+1}-\theta_t)|\leq \frac{\beta}{2}\|\theta_{t+1}-\theta_t\|^2.
\]
When $\theta_{t+1}=\theta_t + \eta \nabla_\theta f(\theta_t)$,
\[
f(\theta_{t+1})\geq f(\theta_t) + \eta \|\nabla_\theta f(\theta_t)\|^2 - \frac{\beta}{2}\|\eta \nabla_\theta f(\theta_t)\|^2,
\]
or equivalently,
\[
f(\theta_{t+1})-f(\theta_t)\geq \eta(1-\frac{\beta\eta}{2})\|\nabla_\theta f(\theta_t)\|^2.
\]
If learning rate $\eta\leq \frac{1}{\beta}$, $f(\theta_{t+1})-f(\theta_t)\geq \frac{\eta}{2}\|\nabla_\theta f(\theta_t)\|^2$. Therefore,
\[
\mathcal{L}_t(\theta_{t+1}) - \mathcal{L}_t(\theta_t) \geq \frac{\eta}{2}\|g_t(\theta_t)\|^2.
\]
As a result,
\[
J(\theta_{t+1}) - J(\theta_t) = J(\theta_{t+1}) - \mathcal{L}_t(\theta_t) \geq \mathcal{L}_t(\theta_{t+1}) - \mathcal{L}_t(\theta_t) \geq \frac{\eta}{2}\|g_t(\theta_t)\|^2 \geq 0.
\]
Since $J(\theta_t)$ is non-decreasing in $t$ and $J(\theta)$ is bounded above by $1/(1-\gamma)$, it must hold that
\[\sum_{t=0}^T \frac{\eta}{2}\|g_t(\theta_t)\|^2 \leq J(\theta_T)-J(\theta_0)\leq \frac{1}{1-\gamma}-J(\theta_0)< \infty,\;\forall T\]
which means $\lim_{t\rightarrow \infty} \|g_t(\theta_t)\|^2 =0 $, and therefore $\lim_{t\rightarrow \infty} g_t(\theta_t)=0$. 
\end{proof}

From Appendix~\ref{app:spg-softmax}, we know
\[
[g_t(\theta_t)](s,a) = d(s)\pi_{\theta_t}(a|s)A^t(s,a).
\]
The following lemma states the magnitude of $A^t(s,a)$ for $a\in I^+_s\cup I^-_s$.
\begin{lemma}{[Lemma 43 of \cite{agarwal2021theory}]}
\label{lemma:A_bound}
Define $A^t(s,a):= Q^t(s,a)-V^t(s)$. There exists a $T_1$ such that for all $t>T_1$, for all states $s$ and actions $a$, we have
\[
\begin{aligned}
A^t(s,a) < -\frac{\Delta}{4},~ \forall a \in I_s^{-};~~~A^t(s,a)>\frac{\Delta}{4}, ~\forall a \in I_s^+ ; ~\mathrm{where}\\
\Delta = \min_{\{s,a|A^\infty(s,a)\neq 0\}} |A^\infty(s,a)|, ~\mathrm{and}~ A^\infty(s,a) = Q^\infty(s,a) - V^\infty(s)
\end{aligned}
\]
\end{lemma}
The proof of Lemma~\ref{lemma:A_bound} follows the same steps as in \cite{agarwal2021theory}.

Since when $t>T_1$, $|A^t(s,a)|>\Delta / 4$ for $a\in I^+_s\cup I^-_s$ (Lemma~\ref{lemma:A_bound}), we must also have that
\[|[g_t(\theta_t)](s,a)|=d(s)\pi_{\theta_t}(s,a)|A^t(s,a)| \geq (\min_{s'} d(s'))\frac{\Delta}{4}\pi_{\theta_t}(s,a).\]
Hence, $[g_t(\theta_t)](s,a)\rightarrow 0$ as $t\rightarrow \infty$ (Lemma~\ref{lemma:gradient-0}), which implies that $\pi_t(a|s)\rightarrow 0$ for $a\in I^+_s\cup I^-_s$.
This completes the proof of this part.

\subsubsection{Part 2}
In this part, we show that $I^+_s =\emptyset$. This is proved by contradiction and follows the same steps as in \cite{agarwal2021theory}. To shorten our discussion we omit to present the proof of the lemmas below since they follow directly the steps used in \cite{agarwal2021theory} with $\nabla_\theta V^\pi(s_0)$ replaced with $g_t(\theta_t)$ and exploiting the vanishing property of $g_t(\theta_t)$ established in Part 1.

\begin{lemma}{[Lemma 47 of \cite{agarwal2021theory}]}
\label{lemma:47}
    Suppose $a_+\in I^+_s$. For any $a\in I^0_s$, if there exists a $t>T_0$, such that $\pi_t(a_+|s)\geq \pi_t(a|s) $, then for all $\tau > t$, $\pi_{\tau}(a_+|s)\geq \pi_\tau(a|s)$.
\end{lemma}
Therefore, for any action $a_+\in I^+_s$, we can partition the set $I^0_s$ into $B^0_s(a_+)$ and $\bar{B}^0_s(a_+)$ such that $\forall a\in B^0_s(a_+)$, we have that $ \pi_t(a_+|s) < \pi_t(a|s)$ for all $t\geq T_0$, and $\bar{B}^0_s(a_+)$ is the complement, i.e. $\forall a\in\bar{B}^0_s(a_+)$, there exists a $t'>T_0$ such that $\pi_t(a_+|s)\geq \pi_t(a|s)$ for all $t\geq t'$. In addition, for any action $a_+\in I^+_s$, the set $B^0_s(a_+)\neq \emptyset$, as stated in the following lemma together with properties about $\pi_t(a|s)$ and $\theta_t(s,a)$ for $a\in B^0_s(a_+)$.
\begin{lemma}{[Lemma 48 of \cite{agarwal2021theory}]}
\label{lemma:48}
    Suppose $I_s^+\neq\emptyset$. For all $a_+\in I^+_s$, we have that $B^0_s(a_+)\neq \emptyset$ and that
    \[\sum_{a\in B^0_s(a_+)} \pi_t(a|s)\rightarrow 1, \mathrm{as}~t\rightarrow \infty\]
    This implies that
    \[\max_{a\in B^0_s(a_+)} \theta_t(s,a)\rightarrow \infty.\]
\end{lemma}

The following lemma also establishes useful properties about the sequence $\theta_t(s,a)$ depending on whether $a$ is in $I_s^+$ or in $I_s^-$.

\begin{lemma}{[Lemma 50 of \cite{agarwal2021theory}]}\label{thm:thetaseq}
For all actions $a_+ \in I_s^+$, the sequence $\theta_t(s,a)$ is bounded from below as
$t \to \infty$. Moreover, for all actions $a_- \in I_s^-$, we have that $\theta_t(s,a) \rightarrow -\infty$ as $t \rightarrow \infty$.
\end{lemma}

Based on Lemma \ref{thm:thetaseq}, we must have that for any $a_-\in I^-_s$, $\frac{\pi(a_-|s)}{\pi(a_+|s)}=\exp(\theta(s,a_-)-\theta(s,a_+))\rightarrow 0$ since $\theta(s,a_+)$ is lower bounded and $\theta(s,a_-)\rightarrow -\infty$. There must therefore exist a $T_2>T_0$ such that for all $t>T_2$
\[ \frac{\pi_t(a_-|s)}{\pi_t(a_+|s)}<\frac{(1-\gamma)\Delta}{16|\mathcal{A}|},\]
or, equivalently,
\begin{equation}
\label{eq:a-}
    -\sum_{a\in I^-_s} \frac{\pi_t(a_-|s)}{1-\gamma} > -\pi_t(a_+|s) \frac{\Delta}{16}.
\end{equation}

For $\bar{a}\in \bar{B}^0_s$, we have $A^t(s,\bar{a})\rightarrow 0$ by the definition of $I^0_s$. From Lemma~\ref{lemma:47}, there must exist a $T_{a_+}$ such that when $t>T_{a_+}$, $\pi_t(a_+|s)\geq \pi_t(\bar{a}|s), \forall \bar{a}\in \bar{B}^0_s$. Thus, there exists a $T_3 > \max(T_2, T_{a_+})$ such that
\[ |A^t(s,\bar{a})|<\frac{\Delta}{16|\mathcal{A}|}\leq\frac{\pi_t(a_+|s)}{\pi_t(\bar{a}|s)}\frac{\Delta}{16|\mathcal{A}|},\;\forall \bar{a}\in \bar{B}^0_s,\;\forall t>T_3,\]
which implies
\[ \sum_{\bar{a}\in \bar{B}^0_s(a_+)} \pi_t(\bar{a}|s) |A^t(s,\bar{a})| < \pi_t(a_+|s) \frac{\Delta}{16},\;\forall t>T_3.\]
We thus obtain
\begin{equation}
\label{eq:a-bar}
    -\pi_t(a_+|s)\frac{\Delta}{16}<\sum_{\bar{a}\in \bar{B}^0_s(a_+)} \pi_t(\bar{a}|s) A^t(s,\bar{a}) < \pi_t(a_+|s) \frac{\Delta}{16},\;\forall t>T_3.
\end{equation}

Putting the different bounds together, we get that for all $t>T_3$,
\[
\begin{aligned}
&0=\sum_{a \in \mathcal{A}} \pi_t(a|s)A^t(s,a)\\
&=\sum_{a\in I^0_s} \pi_t(a|s)A^t(s,a)+\sum_{a_+\in I^+_s} \pi_t(a_+|s)A^t(s,a_+) + \sum_{a_-\in I^-_s} \pi_t(a_-|s)A^t(s,a_-)\\
&\overset{(a)}{\geq} \sum_{a\in B^0_s(a_+)} \pi_t(a|s)A^t(s,a) + \sum_{\bar{a}\in \bar{B}^0_s(a_+)} \pi_t(\bar{a}|s)A^t(s,\bar{a}) + \pi_t(a_+|s) A^t(s,a_+)\\
&~~~~~~~~~~~~~~~~~~~~~~~~~~~~~~~~~~~~~~~~~~~~~+ \sum_{a_-\in I^-_s}\pi_t(a_-|s)A^t(s,a_-)\\
&\overset{(b)}{\geq} \sum_{a\in B^0_s(a_+)}\pi_t(a|s)A^t(s,a) + \sum_{\bar{a}\in \bar{B}^0_s(a_+)}\pi_t(\bar{a}|s)A^t(s,\bar{a}) + \pi_t(a_+|s)\frac{\Delta}{4} - \sum_{a_-\in I^-_s} \frac{\pi_t(a_-|s)}{1-\gamma}\\
&\overset{(c)}{>}\sum_{a\in B^0_s} \pi_t(a|s)A^t(s,a) -\pi_t(a_+|s)\frac{\Delta}{16} + \pi^t(a_+|s)\frac{\Delta}{4} - \pi_t(a_+|s)\frac{\Delta}{16}\\
&> \sum_{a\in B^0_s(a_+)} \pi_t(a|s)A^t(s,a),
\end{aligned}
\]
where (a) uses $A^t(s,a_+)>0, \forall a_+\in I^+_s$ and only one $a_+$ is chosen; (b) uses $A^t(s,a_+)\geq \frac{\Delta}{4}$ for $t > T_1$ from Lemma~\ref{lemma:A_bound}, and $A^t(s,a_-)\geq -\frac{1}{1-\gamma}$; (c) uses Eq.~\ref{eq:a-} and left inequality in Eq.~\ref{eq:a-bar}. This implies that for all $t > T_3$
\begin{equation}
\label{eq:grad_a}
    \sum_{a\in B^0_s(a_+)} [g_t(\theta_t)](s,a) = \sum_{a\in B^0_s(a_+)} d(s)\pi_t(a|s)A^t(s,a) < 0.
\end{equation}

Since for $a\in B^0_s(a_+), \pi_t(a_+|s)<\pi_t(a|s)$ for all $t>T_0$. This implies by the softmax parameterization that $\theta_t(s,a_+)<\theta_t(s,a)$. Since $\theta_t(s,a_+)$ is lower bounded (Lemma \ref{thm:thetaseq}), this implies $\theta_t(s,a)$ is also lower bounded. This in conjunction with $\max_{a\in B^0_s(a_+)}\theta_t(s,a)\rightarrow \infty$ (Lemma~\ref{lemma:48}) implies that $\sum_{a\in B^0_s(a_+)}\theta_t(s,a)\rightarrow \infty$. This contradicts with Eq.~\ref{eq:grad_a} that the sum of updates for these same parameters is negative for all $t>T_3$. This contradiction completes the proof that $I^+_s=\emptyset$.

\subsection{Joint Elicitability of CVaR}
\label{app:joint-elic-cvar}
Corollary 5.5 of \cite{fissler2016higher} gives a general form of loss function for $(\mathrm{VaR}_\alpha, \mathrm{CVaR}_\alpha)$,
\[
\begin{aligned}
&\mathrm{VaR}_\alpha(\tilde{x}),\mathrm{CVaR}_\alpha(\tilde{x})=\arg\min_{y_1,y_2}\mathbb{E}[L(\tilde{x},y_1,y_2)], ~\mathrm{with}\\
&L(x,y_1,y_2)=\Big(\mathbb{I}_{\{x\leq y_1\}}-\alpha\Big)G_1(y_1) - \mathbb{I}_{\{x\leq y_1\}} G_1(x) + G_2(y_2)\Big(y_2-y_1+\frac{1}{\alpha} \mathbb{I}_{\{x\leq y_1\}}(y_1-x)\Big) - G_3(y_2) + \phi(x),
\end{aligned}
\]
where $G_1$ is increasing, $G'_3=G_2$, $G_3$ is strictly increasing and strictly convex. The partial derivatives of $L$ are
\[
\begin{aligned}
\frac{\partial L}{\partial y_1}&= (\mathbb{I}\{x\leq y_1\}-\alpha)\big(G'_1(y_1) + \frac{G_2(y_2)}{\alpha}\big)\\
\frac{\partial L}{\partial y_2}&=G'_2(y_2)\big(y_2-y_1 + \frac{1}{\alpha}\mathbb{I}\{x\leq y_1\}(y_1-x)\big).
\end{aligned}
\]
Note that in $\frac{\partial L}{\partial y_1}$, the $\mathbb{I}\{x\leq y_1\}-\alpha$ term is the gradient of quantile regression. But it also contains a term that depends on the value of $y_2$ (which is the CVaR estimator). In $\frac{\partial L}{\partial y_2}$, $\mathbb{E}[y_2 - y_1 + \frac{1}{\alpha}\mathbb{I}\{\tilde{x}\leq y_1\}(y_1-\tilde{x})] = y_2 - \mathrm{CVaR}_\alpha(\tilde{x})$ if $y_1=\mathrm{VaR}_\alpha(\tilde{x})$, which directly quantifies the difference between $y_2$ and $\mathrm{CVaR}_\alpha(\tilde{x})$, but it also contains a term  $G'_2(y_2)$. This motivates us to investigate if those additional terms can be removed so as to improve the stability of gradient updates.

In practice, quantile is elicitable via quantile regression loss, and $\alpha$-CVaR can be written as an expectation if $\alpha$-quantile is known (cf. Eq.~\ref{eq:def-cvar}), i.e.,
\[
\begin{aligned}
    \mathrm{VaR}_\alpha(\tilde{x}) &= \arg\min_{y_1} \mathbb{E}[l_\alpha(\tilde{x}-y_1)],~~l_\alpha(x-y):=(\alpha-\mathbb{I}\{x< y\})(x-y)\\
    \mathrm{CVaR}_\alpha(\tilde{x}) &= \arg\min_{y_2}\frac12\Big(y_2 - \big(\mathrm{VaR}_\alpha-\frac{1}{\alpha}\mathbb{E}[(\mathrm{VaR}_\alpha-\tilde{x})_+]\big)\Big)^2.
\end{aligned}
\]
This suggests to update $y_1$ and $y_2$ via
\[
\begin{aligned}
    y_1&\leftarrow y_1 - \zeta_1 \cdot \partial_{y_1} \mathbb{E}[l_\alpha(\tilde{x}-y_1)] = y_1 - \zeta_1 \cdot \mathbb{E}[\mathbb{I}\{\tilde{x}< y_1\}-\alpha]\\
    y_2 &\leftarrow y_2 - \zeta_2\cdot \big(y_2 - (y_1 - \frac{1}{\alpha}\mathbb{E}[(y_1 - \tilde{x})_+])\big).
\end{aligned}
\]
When using a single sample $x$ of $\tilde{x}$, the update rules become
\[
\begin{aligned}
    y_1&\leftarrow  y_1 - \zeta_1 \cdot (\mathbb{I}\{\tilde{x}< y_1\}-\alpha)\\
    y_2 &\leftarrow y_2 - \zeta_2\cdot \big(y_2 - y_1 + \frac{1}{\alpha}(y_1 - \tilde{x})_+\big).
\end{aligned}
\]
Comparing with the gradient of the joint elicitable loss function, the terms $G'_1(y_1)+G_2(y_2)/\alpha$ and $G'_2(y_2)$ are removed. We will show in Proposition \ref{prop:cvar} that proper convergence guarantees of this learning procedure can be obtained using two-time-scale stochastic optimization arguments.

\subsection{Proof of Proposition~\ref{prop:expectile}}
The convergence analysis follows the framework presented in Sec. 4 of \cite{bertsekas1996neuro} and is similar to \cite{shen2014risk}. We first recall the standard operator convergence results in \cite{bertsekas1996neuro}.

Consider the update rule
\begin{equation}
\label{eq:general-update}
    x_{t+1}(i) = (1-\beta_t(i))x_t(i) + \beta_t(i)[(Hx_t)(i)+w_t(i)],
\end{equation}

where $x_t\in\mathbb{R}^d$, $H:\mathbb{R}^d\rightarrow \mathbb{R}^d$ is an operator, $w_t$ denotes some random noise and $\beta_t$ is learning rate. Denote by $\mathcal{F}_t$ the history of the update up to time $t$,
\[
\mathcal{F}_t:=\{x_0,...,x_t, w_0,...,w_t,\beta_0,...,\beta_t\}
\]
\begin{proposition}{[Proposition 4.4 of \cite{bertsekas1996neuro}]}
\label{prop:single-time-scale-convergence}
    Let $x_t$ be the sequence generated by the iteration~\ref{eq:general-update}. Assume that
    
    1. The learning rates $\beta_t(i)$ are non-negative and satisfy
    \[
    \sum_{t=0}^\infty \beta_t(i)=\infty, ~\sum_{t=0}^\infty \beta^2_t(i)<\infty, ~\forall i
    \]

    2. The noise terms $w_t(i)$ satisfy 1) for all $i$ and $t$, $\mathbb{E}[w_t(i)|\mathcal{F}_i]=0$; 2) There exists a norm $||\cdot||$ on $\mathbb{R}^d$ and $a,b\in\mathbb{R}$ such that
    \[
    \mathbb{E}[w_t^2(i)|\mathcal{F}_t]\leq a + b ||Q_t||^2.
    \]
    
    3. The operator $H$ in~\ref{eq:general-update} is a contraction under sup-norm.

    Then $x_t$ converges to $x^*$, a fixed point of $H$ (i.e., $Hx^*=x^*$) with probability 1.
\end{proposition}

To prove Proposition~\ref{prop:expectile}, we first reformulate the updating rule in Eq.~\ref{eq:expectile-qlearning} as
\[
Q_{t+1}(s,a) = Q_t(s,a)+2\zeta_t(s,a)\cdot u(r+\gamma V_t(s')-Q_t(s,a)),
\]
where $u(x)=(1-\alpha)(x)_{-}+\alpha(x)_+$ with $\alpha\in(0,1)$. It holds that $0<\epsilon<\frac{u(x)-u(y)}{x-y}\leq L$ with $\epsilon:=\min(\alpha,1-\alpha)$ and $L:=\max(\alpha,1-\alpha)$ when $x\neq y$. 

To align with Eq.~\ref{eq:general-update}, the update rule can be further expressed as
\[
Q_{t+1}(s,a) = (1-\frac{2\zeta_t(s,a)}{l})Q_t(s,a) + \frac{2\zeta_t(s,a)}{l}[l\cdot u(r+\gamma V_t(s')-Q_t(s,a))+Q_t(s,a)],
\]
where $l\in(0,\min(L^{-1},1)]$ (The range of $l$ is chosen to obtain the contraction of $H^\pi$ in the following).

With this formulation, define the operator $H^\pi$ as 
\begin{equation}
\label{eq:H-expectile}
\begin{aligned}
    (H^\pi Q)(s,a)&=l\cdot \sum_{r,s'}P(r,s'|s,a)u(r+\gamma V(s')-Q(s,a)) + Q(s,a)\\
    &\mathrm{with}~V(s')=\sum_{a'}\pi(a'|s')Q(s',a'),
\end{aligned}
\end{equation}
and noise $w_t$ as
\begin{equation}
\label{eq:noise-expectile}
w_t(s,a) = l \cdot u(r+\gamma V_t(s')-Q_t(s,a))-l\cdot\mathbb{E}_{r,s'} \big[u(r+\gamma V_t(s')-Q_t(s,a))\big].
\end{equation}

\textbf{First}, show that $H^\pi$ in Eq.~\ref{eq:H-expectile} is a contraction. We start with
\[
\begin{aligned}
|V_1(s) - V_2(s)| &= |\sum_a \pi(a|s)(Q_1(s,a)-Q_2(s,a))|\leq \sum_a \pi(a|s)|Q_1(s,a)-Q_2(s,a)|\\
&\leq \sum_a \pi(a|s)\|Q_1-Q_2\|_\infty = \|Q_1-Q_2\|_\infty.
\end{aligned}
\]
Next, we can show
\[
\begin{aligned}
(H^\pi Q_1)(s,a) - (H^\pi Q_2)(s,a)=&l\cdot \sum_{r,s'}P(r,s'|s,a)\big(u(r+\gamma V_1(s')-Q_1(s,a))-u(r+\gamma V_2(s')-Q_2(s,a))\big)\\
&+ Q_1(s,a) - Q_2(s,a)\\
=& l\cdot \sum_{r,s'} P(r,s'|s,a)\xi_{(s,a,r,s',Q_1,Q_2)}\big(\gamma V_1(s')-\gamma V_2(s')-Q_1(s,a)+Q_2(s,a)\big)\\
&+ Q_1(s,a) - Q_2(s,a)\\
\overset{(a)}{=}&l\gamma \sum_{r,s'}P(r,s'|s,a)\xi_{(s,a,r,s',Q_1,Q_2)} (V_1(s')-V_2(s'))\\
&+\big(1-l\sum_{r,s'}P(r,s'|s,a)\xi_{(s,a,r,s',Q_1,Q_2)}\big)\big(Q_1(s,a)-Q_2(s,a)\big)\\
\overset{(b)}{\leq}&l\gamma \sum_{r,s'}P(r,s'|s,a)\xi_{(s,a,r,s',Q_1,Q_2)} \|Q_1-Q_2\|_\infty\\
&+\big(1-l\sum_{r,s'}P(r,s'|s,a)\xi_{(s,a,r,s',Q_1,Q_2)}\big)\|Q_1-Q_2\|_\infty\\
=& \big(1-l(1-\gamma)\sum_{r,s'}P(r,s'|s,a)\xi_{(s,a,r,s',Q_1,Q_2)}\big)\|Q_1-Q_2\|_\infty\\
\leq& (1-l(1-\gamma)\epsilon)\|Q_1-Q_2\|_\infty,
\end{aligned}
\]
where (a) is due to $0<\epsilon < \frac{u(x)-u(y)}{x-y}\leq L$, so there exists $\xi_{(x,y)}\in(\epsilon,L]$ such that $u(x)-u(y)=\xi_{(x,y)}(x-y)$; (b) is due to $V_1(s')-V_2(s')\leq\|Q_1-Q_2\|_\infty$ and $\big(1-l\sum_{r,s'}P(r,s'|s,a)\xi_{(s,a,r,s',Q_1,Q_2)}\big)\geq 0$ when $l\in(0,\min(L^{-1},1)]$. Therefore, $H^\pi$ is a contraction.

$Q^\pi$ is the fixed point of $H^\pi$, since when $Q=Q^\pi$, the gradient of the loss function $(1-\alpha)\mathbb{E}\big[(\tilde{r}+\gamma V(\tilde{s}')-y)^2_{-}\big] + \alpha \mathbb{E}\big[(\tilde{r}+V(\tilde{s}')-y)^2_{+}\big]$ is zero, which is $\mathbb{E}_{r,s'}[u(r+\gamma V(s')-Q(s,a))]=0$. As a result, Eq.~\ref{eq:H-expectile} remains $H^\pi Q^\pi = Q^\pi$.

\textbf{Second}, check the conditions of the noise term $w_t(s,a)$ in Eq.~\ref{eq:noise-expectile}. $\mathbb{E}[w_t(s,a)|\mathcal{F}_t]=0$ by its definition, thus condition 1) holds. Denote $\delta_t(s,a,r,s')=r + \gamma V_t(s')-Q_t(s,a)$,
\[
\mathbb{E}\big[w_t^2(s,a)|\mathcal{F}_t\big]=l^2\mathbb{E}_{r,s'}\big[u^2(\delta_t(s,a,r,s'))\big] - l^2\big(\mathbb{E}_{r,s'}\big[u(\delta_t(s,a,r,s'))\big]\big)^2\leq l^2\mathbb{E}_{r,s'}\big[u^2(\delta_t(s,a,r,s'))\big]
\]
Let $\bar{R}$ be the bound of $|\tilde{r}|$. Then $|\delta_t(s,a,r,s')|\leq \bar{R} + 2\|Q_t\|_\infty$, which implies $|u(\delta_t(s,a,r,s'))-u(0)|\leq L|\delta_t(s,a,r,s')-0|\leq L(\bar R + 2\|Q_t\|_\infty)$. Therefore
\[
\big(u(\delta_t(s,a,r,s'))\big)^2=\big(\big|u(\delta_t(s,a,r,s'))-u(0)\big|\big)^2\leq \big(L\bar R + 2L\|Q_t\|_\infty\big)^2\leq 2\big(L\bar R\big)^2 + 8L^2\|Q_t\|^2_\infty,
\]
where the last inequality is due to Cauchy-Schwarz inequality $(x+y)^2\leq 2x^2+2y^2$. As a result
\[
\mathbb{E}\big[w^2_t(s,a)|\mathcal{F}_t\big]\leq l^2\cdot\mathbb{E}_{r,s'}\big[u^2(\delta_t(s,a,r,s'))\big] \leq 2l^2\big(L\bar R\big)^2 + 8l^2L^2\|Q_t\|^2_\infty.
\]
Hence, condition 2) holds.

The contraction of $H^\pi$ and the conditions being satisfied by $w_t(s,a)$ indicate $Q_t(s,a)\rightarrow Q^\pi(s,a)$, which completes the proof for Proposition~\ref{prop:expectile}.

\subsection{Proof of Proposition~\ref{prop:cvar}}
The convergence analysis follows the two-time scale framework presented in Chapter 8 of \cite{borkar2023stochastic}. We first recall this convergence results.

Consider the iterations
\begin{equation}
\label{eq:borkar2008}
\begin{aligned}
x_{n+1} &= x_n + a_n [h(x_n,y_n) + M_{n+1}^{(1)}],\\
y_{n+1} &= y_n + b_n [g(x_n,y_n) + M_{n+1}^{(2)}],
\end{aligned}
\end{equation}
where $h:\mathbb{R}^{d_x+d_y}\rightarrow \mathbb{R}^{d_x}, g:\mathbb{R}^{d_x+d_y}\rightarrow \mathcal{R}^{d_y}$ ($d_x$ and $d_y$ represent the dimension of $x$ and $y$) are Lipschitz, and $\{M_n^{(1)}\},\{M_n^{(2)}\}$ are martingale difference sequences with respect to the increasing $\sigma$-fields
$$
\mathcal{F}_n := \sigma(x_m, y_m,M^{(1)}_m,M^{(2)}_m, m\leq n), n\geq0,
$$
satisfying
$$
\mathbb{E}[\|M^{(i)}_{n+1}\|^2 |\mathcal{F}_n] \leq K(1+\|x_n\|^2 + \|y_n\|^2), i=1,2.
$$
Stepsizes $\{a_n\},\{b_n\}$ are positive scalars satisfying 
$$
\sum_t a_n = \sum_t b_n =\infty, \sum_t(a_n^2 + b_n^2)<\infty, \frac{b_n}{a_n}\rightarrow 0. 
$$

\begin{proposition}{[Theorem 8.1 in Chapter 8 of \cite{borkar2023stochastic}]}
\label{prop:two-time-scale}
    Let $x_n$, $y_n$ be the sequences generated by iterations~\ref{eq:borkar2008}. Assume that
    
    1. $\dot{x}(t)=h(x(t),y)$ has a globally asymptotically stable equilibrium $\lambda(y)$, where $\lambda:\mathbb{R}^{d_y}\rightarrow \mathbb{R}^{d_x}$ is a Lipschitz map.

    2. $\dot{y}(t) = g(\lambda(y(t)), y(t))$ has a globally asymptotically stable equilibrium $y^*$.

    3. $\sup_n(\|x_n\|+\|y_n\|)<\infty$ \EDmodified{almost surely}.

    Then $(x_n,y_n)\rightarrow (\lambda(y^*),y^*)$ almost surely.
\end{proposition}

We reformulate the update rule for $q$ and $Q$ in Eq.~\ref{eq:cvar-qlearning} to align with Eq.~\ref{eq:borkar2008}. To distinguish with the $t$ in ODE, we use $n$ as the subscript for $q$ and $Q$ instead.
$$
\begin{aligned}
q_{n+1}(s,a) &= q_n(s,a) + \zeta_n^q(s,a)\big(h_{s,a}(q_n,Q_n) + M^q_{n+1}(s,a)\big),~\mathrm{with}\\
h_{s,a}(q_n,Q_n)&=-\Big(\mathbb{P}\big(\tilde{r}(s,a)+\gamma V_n(\tilde{s}')< q_n(s,a)\big)-\alpha\Big),\\
M_{n+1}^q(s,a) &=  \mathbb{P}\big(\tilde{r}(s,a)+\gamma V_n(\tilde{s}')< q_n(s,a)\big)-\mathbb{I}\{r+\gamma V_n(s')< q_n(s,a)\}.
\end{aligned}
$$

$$
\begin{aligned}
Q_{n+1}(s,a) &= Q_n(s,a) + \zeta_n^Q(s,a)(g_{s,a}(q_n,Q_n)+M^Q_{n+1}(s,a)), ~\mathrm{with}\\
g_{s,a}(q_n,Q_n) &= -\Big(Q_n(s,a)-q_n(s,a)+\frac{1}{\alpha}\mathbb{E}_{r,s'}\big[(q_n(s,a)-r-\gamma V_n(s'))_+\big]\Big),\\
M_{n+1}^Q(s,a) &=  \frac{1}{\alpha}\mathbb{E}_{r,s'}\big[(q_n(s,a)-r-\gamma V_n(s'))_+\big] - \frac{1}{\alpha}(q_n(s,a)-r-\gamma V_n(s'))_+,
\end{aligned}
$$
where $V_n(s')=\sum_{a'}\pi(a'|s')Q_n(s,a)$. 

\textbf{First}, check the Lipschitz of $h$ and $g$. 

For $h$, when $Q$ is fixed, since we assume reward has bounded density, $f_{\tilde{r}(s,a)+\gamma V(\tilde{s}')}(z) < M$, then
\[
\begin{aligned}
|h_{s,a}(q_1,Q)-h(q_2,Q)|&=|F_{\tilde{r}(s,a)+
\gamma V(\tilde{s}')}(q_2(s,a))-F_{\tilde{r}(s,a)+
\gamma V(\tilde{s}')}(q_1(s,a))|\\
&=\Big|\int^{q_2(s,a)}_{q_1(s,a)} f_{\tilde{r}(s,a)+\gamma V(\tilde{s}')}(z) dz\Big|\\
&\leq M|q_1-q_2|,
\end{aligned}
\]
where $F$ is the CDF and $f$ is the density function.

When $q$ is fixed, 
\[
\begin{aligned}
|h_{s,a}(q,Q_1)-h_{s,a}(q,Q_2)| &= |\mathbb{P}(\tilde{r}(s,a)+\gamma V_1(\tilde{s}')<q(s,a)) - \mathbb{P}(\tilde{r}(s,a)+\gamma V_2(\tilde{s}')<q(s,a))|\\
&=|F_{\tilde{r}(s,a)+\gamma V_1(\tilde{s}')}(q(s,a))-F_{\tilde{r}(s,a)+\gamma V_2(\tilde{s}')}(q(s,a))|
\end{aligned}
\]
Denote $\tilde{z}_1:=\tilde{r}(s,a)+\gamma V_1(\tilde{s}')$, and $\tilde{z}_2:=\tilde{r}(s,a)+\gamma V(\tilde{s}') $, we have $|\tilde{z}_1-\tilde{z}_2|=|\gamma V_1(\tilde{s}')-\gamma V_2(\tilde{s}')|\leq \|Q_1-Q_2\|_\infty:= \epsilon$ almost surely. Thus $\tilde{z}_2 < q-\epsilon$ ensures $\tilde{z}_1 < q$, which implies $F_{\tilde{z}_2}(q-\epsilon)\leq F_{\tilde{z}_1}(q)$. By the same logic, we have $F_{\tilde{z}_1}(q)\leq F_{\tilde{z}_2}(q+\epsilon)$. This leads to
\[
|F_{\tilde{z}_1}(q(s,a))-F_{\tilde{z}_2}(q(s,a))| \leq|F_{\tilde{z}_2}(q(s,a)+\epsilon)-F_{\tilde{z}_2}(q(s,a))|\leq M\|Q_1-Q_2\|_\infty.
\]
As a result, 
\[
\begin{aligned}
|h_{s,a}(q_1,Q_1)-h_{s,a}(q_2,Q_2)|&\leq |h_{s,a}(q_1,Q_1)-h_{s,a}(q_2,Q_1)|+|h_{s,a}(q_2,Q_1)-h_{s,a}(q_2,Q_2)|\\
&\leq M|q_1-q_2| + M\|Q_1-Q_2\|_\infty
\end{aligned}
\]

For $g$, we need to show its third term, i.e., $k_{s,a}(q,Q):=\mathbb{E}_{r,s'}[(q(s,a)-r-\gamma V(s'))_+]$ is Lipschitz.
\[
\begin{aligned}
    |k_{s,a}(q_1,Q_1)-k_{s,a}(q_2,Q_2)|=&|\mathbb{E}[(q_1(s,a)-r-\gamma V_1(s'))_+]-\mathbb{E}[(q_2(s,a)-r-\gamma V_2(s'))]|\\
    \leq& \mathbb{E}[|q_1-q_2 - \gamma(V_1(s')-V_2(s'))|]\\
    \leq& |q_1-q_2| + \gamma \mathbb{E}[|V_1(s')-V_2(s')|]\\
    \leq& |q_1-q_2| + \|Q_1-Q_2\|_\infty
\end{aligned}
\]

\textbf{Second}, check the conditions of the difference terms. 

$|M_{n+1}^q(s,a)|\leq 1$, thus $\mathbb{E}[M_{n+1}^q(s,a)^2|\mathcal{F}_n]\leq 1$. Denote $\delta_n(s,a,r,s')=q_n(s,a)-r-\gamma V_n(s')$, 
$$
\mathbb{E}[M_{n+1}^Q(s,a)^2|\mathcal{F}_n]=\frac{1}{\alpha^2}\mathbb{E}_{r,s'}[(\delta_n(s,a,r,s'))_+^2]-\frac{1}{\alpha^2}\big(\mathbb{E}_{r,s'}\big[(\delta_n(s,a,r,s'))_+\big]\big)^2\leq\frac{1}{\alpha^2}\mathbb{E}_{r,s'}\big[\big(\delta_n(s,a,r,s')\big)_+^2\big].
$$
Let $\bar R$ be the bound of $|\tilde{r}|$. Since $\big(\delta_n(s,a,r,s')\big)_+\leq|\delta_n(s,a,r,s')|=|q_n(s,a)-r-\gamma V(s')|\leq \|q_n\|_\infty + \bar R + \|Q_n\|_\infty$, 
$$
\mathbb{E}[M_{n+1}^Q(s,a)^2|\mathcal{F}_n]\leq \frac{1}{\alpha^2}\mathbb{E}_{r,s'}\big[\big(\delta_n(s,a,r,s')\big)_+^2\big]\leq \frac{3}{\alpha^2}(\bar R^2+\|q_n\|^2_\infty+\|Q_n\|_\infty^2),
$$
where the last inequality follows Cauchy-Schwarz inequality $(x + y +z)^2 \leq 3 (x^2 + y^2 + z^2)$.

Therefore, there exists $K$, such that $\mathbb{E}[M_{n+1}(s,a)^2|\mathcal{F}_n]\leq K(1+\|q_n\|^2_\infty + \|Q_n\|^2_\infty)$ for both $M^q_{n+1}$ and $M^Q_{n+1}$.

\textbf{Third}, check the other three conditions in Proposition~\ref{prop:two-time-scale}. Recall that we assume the random variable is continuous and the quantile is unique in Sec.~\ref{sec:prelim}.

1) When $Q$ is fixed (therefore $V$ is also fixed), for $\EDmodified{\dot{q}_{s,a}(t)}=-\Big(\mathbb{P}\big(\tilde{r}(s,a)+\gamma V(\tilde{s}')\leq \EDmodified{q_{s,a}(t)}\big)-\alpha\Big)$, setting $\EDmodified{\dot{q}_{s,a}(t)}=0$ results in the equilibrium $q^*(s,a)=\mathrm{VaR}_\alpha(\tilde{r}(s,a)+\gamma V(\tilde{s}'))$. Denote $\tilde{z}:=\tilde{r}(s,a)+\gamma V(\tilde{s}')$. It is clear that when $\EDmodified{q_{s,a}(t)} < q^*(s,a)$, $\EDmodified{\dot{q}_{s,a}(t)}>0$ and when $\EDmodified{q_{s,a}(t)}>q^*(s,a)$, $\EDmodified{\dot{q}_{s,a}(t)}<0$.  Consider the case when $\EDmodified{q_{s,a}(0)}<q^*(s,a)$. Since $\EDmodified{\dot{q}_{s,a}(t)}>0$, and $\EDmodified{q_{s,a}(t)}$ is upper bounded by $q^*(s,a)$, we have $\lim_{t\rightarrow \infty} \EDmodified{q_{s,a}(t)}=\bar{q}(s,a)$ with $\bar{q}(s,a)\leq q^*(s,a)$. Suppose $\bar{q}(s,a)<q^*(s,a)$, then $\exists \epsilon >0$, $\EDmodified{\dot{q}_{s,a}(t)}=\alpha - F_{\tilde{z}}\big(\EDmodified{q_{s,a}(t)}\big)>\alpha-F_{\tilde{z}}\big(\bar{q}(s,a)\big)>\epsilon$, where $F_{\tilde{z}}$ is a CDF \EDmodified{with unique quantiles}. Therefore, for all $t\geq 0$, $\EDmodified{q_{s,a}(t)}\geq \EDmodified{q_{s,a}(0)} + \epsilon \EDmodified{t}$. Since $\lim_{t\rightarrow \infty} \big(\EDmodified{q_{s,a}(0)}+\epsilon \EDmodified{t}\big)\rightarrow \infty$, this \EDmodified{contradicts the fact that} $\lim_{t\rightarrow \infty}\EDmodified{q_{s,a}(t)}=\bar{q}(s,a)\leq q^*(s,a)$. Hence, it can only be the case that $\lim_{t\rightarrow \infty}\EDmodified{q_{s,a}(t)}=\EDmodified{\bar{q}(s,a)}=q^*(s,a)$. The same argument holds when $\EDmodified{q_{s,a}(0)}>q^*(s,a)$. This indicates the global attractivity \EDmodified{of $q^*(s,a)$}.

For the Lyapunov stability, define the Lyapunov function 
\[
\mathscr{L}(q) = \int^q_{q^*} (F_{\tilde{z}}(z)-\alpha) dz.
\]
We have $\mathscr{L}(q)\geq 0$ and $\mathscr{L}(q)=0$ only if $q=q^*$, and $\dot{\mathscr{L}}(\EDmodified{q_{s,a}(t)})=-(F_{\EDmodified{\tilde{z}}}(\EDmodified{q_{s,a}(t)})-\alpha)^2 \leq 0$ \EDmodified{for all $t$}, which verifies the stability.

The global attractivity \EDmodified{of $q^*$} and Lyapunov stability together imply that $q^*$ is globally asymptotically stable.

For the $\lambda$ function, in our case,
\[
\lambda_{s,a}(Q)=\mathrm{VaR}_\alpha(\tilde{r}(s,a)+\gamma \mathbb{E}_{a'\sim\pi(\cdot|\tilde{s}')}[Q(\tilde{s}',a')]).
\]
we have \EDmodified{for all $(s,a)$}
\[
\begin{aligned}
&|\lambda_{s,a}(Q_1)-\lambda_{s,a}(Q_2)|\\
=&|\mathrm{VaR}_\alpha(\tilde{r}(s,a)+\gamma \mathbb{E}_{a'\sim\pi(\cdot|\tilde{s}')}[Q_1(\tilde{s}',a')])-\mathrm{VaR}_\alpha(\tilde{r}(s,a)+\gamma \mathbb{E}_{a'\sim\pi(\cdot|\tilde{s}')}[Q_2(\tilde{s}',a')])|\\
\leq&\EDmodified{\max_{s'}}|\gamma \mathbb{E}_{a'\sim\pi(\cdot|\EDmodified{s'})}[Q_1(\EDmodified{s'},a')]-\gamma \mathbb{E}_{a'\sim\pi(\cdot|\EDmodified{s'})}[Q_2(\EDmodified{s'},a')]|\\
\leq&\gamma\|Q_1-Q_2\|_\infty,
\end{aligned}
\]
which means \EDmodified{$\lambda(\cdot)$} is Lipschitz.

2) When $\EDmodified{q_{s,a}(t)}=\mathrm{VaR}_\alpha(\tilde{r}(s,a)+\gamma \EDmodified{V_{\tilde{s}'}(t)})$, 
\begin{equation}\label{eq:odeQ}
\EDmodified{\dot{Q}_{s,a}(t)= q_{s,a}(t)-\frac{1}{\alpha}\mathbb{E}_{r,s'}[(q_{s,a}(t)-r-\gamma V_{s'}(t))_+] - Q_{s,a}(t)= (\mathcal{T}^\pi Q(t))(s,a) -Q_{s,a}(t)}.
\end{equation}
Setting $\dot{Q}\EDmodified{(t)}=0$, the equilibrium \EDmodified{satisfies} $\mathcal{T}^\pi Q=Q$, which is $Q^\pi$. Showing $Q^\pi$ is globally asymptotically stable relies on the following theorem in \cite{borkar2023stochastic}
\begin{theorem}[Theorem 12.1 in Chapter 12 of \cite{borkar2023stochastic}]
\label{thm:gase}
    Consider iterations of the type
    \[
    x_{n+1} = x_n + a_n[F(x_n)-x_n + M_{n+1}].
    \]
    For $x=[x_1,...,x_d]\in\mathbb{R}^d$ and $w=[w_1,...,w_d]$, define
    \[
    \|x\|_{w,p}\overset{def}{=}\Big(\sum_{i=1}^d w_i|x_i|^p\Big)^{1/p}, 1\leq p<\infty,
    \]
    and
    \[
    \|x\|_{w,\infty}\overset{def}{=}\max_i(w_i|x_i|).
    \]
    Assume that 
    \[\|F(x)-F(y)\|_{w,p}\leq \alpha \|x-y\|_{w,p}, \forall x,y\in \mathbb{R}^d, ~~(\mathrm{contraction})\]
    for some $w$, $1\leq q<\infty$ and $\alpha\in(0,1)$.
    
    With the contraction of $F$, if the limiting ODE is $\dot{x}(t)=F(x(t))-x(t)$, with unique $x^*$ such that $x^*=F(x^*)$, then $x^*$ is the globally asymptotically stable equilibrium.
\end{theorem}

In our case, $\mathcal{T^\pi}$ is a contraction under the sup-norm and $Q^\pi$ is the unique solution of $Q=\mathcal{T}^\pi(Q)$. Therefore, according to Theorem~\ref{thm:gase}, $Q^\pi$ is the globally asymptotically stable equilibrium \EDmodified{of \eqref{eq:odeQ}}.

3) $\sup_t(\|q_t\|_\infty+\|Q_t\|_\infty) <\infty $ holds \EDmodified{almost surely} by our assumption.

All the above analysis completes the proof for Proposition~\ref{prop:cvar}.

\subsection{Discussion on the Expectile Bellman Operator in Jiang et al. (2024)}
\label{app:aaai2024}
For a policy $\pi(\cdot|s)$, the Expectile Bellman operator proposed by \cite{jiang2024learning} is
\[
\mathcal{T}^\pi V(s)=V(s) + 2\zeta\mathbb{E}_{a\sim \pi(\cdot|s)}\mathbb{E}_{r,s'}\big[(1-\alpha)\big(\delta(s,a,r,s')\big)_{-}+\alpha\big(\delta(s,a,r,s')\big)_+ \big],
\]
where $\delta(s,a,r,s')=r(s,a) + \gamma V(s') - V(s)$. The contraction of this Bellman operator is established in \cite{jiang2024learning}, but its fixed point is not explicitly characterized.

For a specific action $a\sim \pi(\cdot|s)$, the policy is updated using the advantage $A^\pi(s,a):=\mathbb{E}_{r,s'}\big[(1-\alpha)\big(\delta(s,a,r,s')\big)_{-}+\alpha\big(\delta(s,a,r,s')\big)_+ \big]$ in \cite{jiang2024learning}. This naturally motivates considering the optimality version of this Bellman operator, i.e.,
\begin{equation}
\label{eq:ppo-expectile-opt-bellman}
    \mathcal{T}^* V(s)=V(s) + \max_a 2\zeta \mathbb{E}_{r,s'}\big[(1-\alpha)\big(\delta(s,a,r,s')\big)_{-}+\alpha\big(\delta(s,a,r,s')\big)_+ \big].
\end{equation}
We show that the fixed point of this optimality Bellman operator is $V^*$, which suggests that the method in \cite{jiang2024learning} optimizes dynamic expectile.
\begin{lemma}
    $V^*$ is the unique solution to the Bellman operator in Eq.~\ref{eq:ppo-expectile-opt-bellman}.
\end{lemma}
\begin{proof}
   Following Eq.~\ref{eq:bellman-opt-eq}, the Bellman optimality equation for $V^*$ is
    \begin{equation}
    \label{eq:v-star-bellman}
        V^*(s) = \max_a Q^*(s,a) = \max_a \rho(\tilde{r}(s,a)+\gamma V^*(\tilde{s}')).
    \end{equation}

    By subtracting $V(s)$ from both side of Eq.~\ref{eq:ppo-expectile-opt-bellman} and dividing by $2\zeta$, we obtain
    \begin{equation}
    \label{eq:ppo-expectile-proof-eq}
         \max_a  \mathbb{E}_{r,s'}\big[(1-\alpha)\big(\delta(s,a,r,s')\big)_{-}+\alpha\big(\delta(s,a,r,s')\big)_+ \big]  =0
    \end{equation}

    Eq.~\ref{eq:v-star-bellman} implies that for all $a$, we have that:
\[
\begin{aligned}
V^*(s)\geq& \rho(\tilde{r}(s,a) + \gamma~V^*(\tilde{s}'))\\
\overset{\mathrm{def}}{=}&\arg\min_y \mathbb{E}[L(\tilde{r}(s,a)+\gamma V^*(\tilde{s}'), y)], ~\forall a\\
=& \arg\min_y (1-\alpha)\mathbb{E}\big[\big(\tilde{r}(s,a)+\gamma V^*(\tilde{s}
)-y\big)^2_{-}\big]+\alpha\mathbb{E}\big[\big(\tilde{r}(s,a)+\gamma V^*(\tilde{s}
)-y\big)^2_{+}\big], ~\forall a
\end{aligned}
\]
which implies that for all $a$
\[
\begin{aligned}
0\leq& \nabla_y \mathbb{E}[L(\tilde{r}(s,a) + \gamma~V^*(\tilde{s}'), y)]|_{y=V^*(s)}\\
=&-2\mathbb{E}[(1-\alpha)(\tilde{r}(s,a)+\gamma V^*(\tilde{s}')-V^*(s))_{-}+\alpha(\tilde{r}(s,a)+\gamma V^*(\tilde{s}')-V^*(s))_{+}],~~~\forall a.
\end{aligned}
\]
Therefore
\begin{equation}
\label{eq:tmp1}
    \max_a \mathbb{E}[(1-\alpha)(\tilde{r}(s,a)+\gamma V^*(\tilde{s}')-V^*(s))_{-}+\alpha(\tilde{r}(s,a)+\gamma V^*(\tilde{s}')-V^*(s))_{+}]\leq 0.
\end{equation}

Moreover, Eq.~\ref{eq:bellman-opt-eq} implies that there exists an optimal action $\pi^*(s)$ such that 
\[
\begin{aligned}
V^*(s)=&\rho(r(s,\pi^*(s)) + \gamma V^*(\tilde{s}'))\\
    =& \arg\min_y \mathbb{E}\big[L\big(\tilde{r}(s,\pi^*(s))+\gamma V^*(\tilde{s}'), y\big)\big].
\end{aligned}
\]
Thus we can confirm that the maximum in Eq.~\ref{eq:tmp1} is $0$ and is achieved when $a=\pi^*(s)$, i.e.,
\[
\mathbb{E}[(1-\alpha)(\tilde{r}(s,\pi^*(s))+\gamma V^*(\tilde{s}')-V^*(s))_{-}+\alpha(\tilde{r}(s,\pi^*(s))+\gamma V^*(\tilde{s}')-V^*(s))_{+}]= 0.
\]
As a result,
\[
\max_a \mathbb{E}[(1-\alpha)(\tilde{r}(s,a)+\gamma V^*(\tilde{s}')-V^*(s))_{-}+\alpha(\tilde{r}(s,a)+\gamma V^*(\tilde{s}')-V^*(s))_{+}] = 0
\]
\end{proof}

\section{Experiment Details}
\label{app:exp-details}
\subsection{Policy and Value Function Parameterization}
For tabular domains (i.e., Maze, CliffWalk), states are represented by the coordinates, i.e. $s:=(x,y)$. Policy parameter is represented by a matrix $\Theta$ with size $[n\_row, n\_column, 4]$. Action probability is calculated by $\pi(a|s)=\frac{\exp(\Theta[x,y,a])}{\sum_b\exp(\Theta[x,y,b])}$. Similarly, value function is also represented by a matrix $\Phi$ with size $[n\_row, n\_column, 4]$ (if learning state-action value) or $[n\_row, n\_column]$ (if learning state value).

For LunarLander, both policy and value functions are implemented by deep neural networks using \texttt{Pytorch} with two hidden layers and \texttt{nn.ReLU()} activation. The hidden layer dimension is 128. Softmax is applied to the output of policy network then feed to \texttt{torch.distributions.Categorical()}. The output dimension of value network is $n\_actions$ (if learning state-action value) or $1$ (if learning state value).

For Inverted Pendulum, both policy and value functions are implemented by deep neural networks using \texttt{Pytorch} with two hidden layers and \texttt{nn.ReLU()} activation. The hidden layer dimension is 128. \texttt{nn.Tanh()} is applied to the output of policy network then multiplied by the maximal magnitude of action to produce the policy mean. Std is a separate learning parameter using \texttt{nn.Parameter()}. The mean and std are feed to \texttt{torch.distributions.Normal()}. The input of value function is state plus action (if learning state-action value) or state only (if learning state value).

Specifically for \cite{coache2023conditionally}, as discussed in Sec.~\ref{sec:related-work}, the original method works on the right tail CVaR. To align with our setting, we replace its value updates by the two-time scale rules as in Eq.~\ref{eq:cvar-qlearning}. Note that this method learns state-only value function.

\subsection{Maze and CliffWalk}
The maze domain is modified from the Github repo [https://github.com/ido90/CeSoR] under MIT license. The cliffwalk domain is implemented by authors.

Discount factor $\gamma=0.999$. For Exp-AC, CVaR-AC, and Exp-PPO, policy learning rate is searched from $\{$5e-4, 1e-3, 5e-3$\}$, value learning rate is searched from $\{$1,2,5$\}$ times the policy learning rate. The learning rates selected are summarized in Table~\ref{tab:maze-lr}.

\begin{table}
  \caption{Learning parameters in Maze and CliffWalk.}
  \label{tab:maze-lr}
  \centering
  \begin{tabular}{lccc}
    \toprule
    Method   & $\pi$ lr        & $Q$ (or $V$) lr & $q$ lr \\
    \midrule
    Exp-AC & 5e-3  & 5e-3 & - \\
    CVaR-AC  & 5e-4 & 1e-3   &  1e-2 \\
    Exp-PPO     & 5e-3       & 1e-2  & -\\
    QL & -& 5e-3 &-\\
    \bottomrule
  \end{tabular}
\end{table}

\subsection{LunarLander}
We use the discrete action version of LunarLander, which is modified from the Github repo [https://github.com/openai/gym] under MIT License. The state dimension is 8, the action space is 4. A detailed description is available at this webpage\footnote{https://gymnasium.farama.org/environments/box2d/lunar\_lander/}.

Discount factor $\gamma=0.99$. Optimizer is Adam. Policy learning rate is searched from $\{$3e-5, 5e-5, 7e-5, 1e-4, 3e-4, 5e-4$\}$. Value learning rate is searched from  $\{$1,2,5$\}$ times the policy learning rate. For Exp-AC, CVaR-AC, and EPG, the entropy coefficient starts from 0.1 and linearly decays to 0.01 in the beginning $15\%$ of steps. The entropy coefficient of Exp-PPO is 0.01. The learning rates selected are summarized in Table~\ref{tab:ll-lr}.

\begin{table}
  \caption{Learning parameters in LunarLander.}
  \label{tab:ll-lr}
  \centering
  \begin{tabular}{lccc}
    \toprule
    Method   & $\pi$ lr        & $Q$ (or $V$) lr & $q$ lr \\
    \midrule
    Exp-AC & 7e-5  & 3.5e-4 & - \\
    CVaR-AC  & 1e-4 & 2e-4   &  1e-3 \\
    Exp-PPO     & 1e-4       & 5e-4  & -\\
    EPG     &   5e-5 & 2.5e-4 &-\\
    \bottomrule
  \end{tabular}
\end{table}

\subsection{Inverted Pendulum}
This domain is modified from the Github repo [https://github.com/openai/gym] under MIT License. The state dimension is 4. The one dimension action is continuous in range $[-3,3]$. A detailed description is available at this webpage\footnote{https://gymnasium.farama.org/environments/mujoco/inverted\_pendulum/}.

Discount factor $\gamma=0.99$. Optimizer is Adam. For each update, Exp-AC, CVaR-AC, and EPG sample $10$ actions to compute policy gradient and the backup value. Policy learning rate is searched from $\{$3e-5, 5e-5, 7e-5, 1e-4, 3e-4, 5e-4$\}$. Value learning rate is searched from  $\{$1,2,5$\}$ times the policy learning rate. For Exp-AC, CVaR-AC, and EPG, the entropy coefficient starts from 0.1 and linearly decays to 0.01 in the beginning $15\%$ of steps. The entropy coefficient of Exp-PPO is 0.01. The learning rates selected are summarized in Table~\ref{tab:ip-lr}.

\begin{table}
  \caption{Learning parameters in Inverted Pendulum.}
  \label{tab:ip-lr}
  \centering
  \begin{tabular}{lccc}
    \toprule
    Method   & $\pi$ lr        & $Q$ (or $V$) lr & $q$ lr \\
    \midrule
    Exp-AC & 5e-5  & 2.5e-4 & - \\
    CVaR-AC  & 5e-5 & 2.5e-4   &  1e-3 \\
    Exp-PPO     & 5e-4       & 2.5e-3  & -\\
    EPG     &   7e-5 & 3.5e-4 &-\\
    \bottomrule
  \end{tabular}
\end{table}

\subsection{Additional Results}
The left landing rate of different methods in LunarLander is shown in Fig.~\ref{fig:ll-left}.

\begin{figure}
    \begin{center}
        \includegraphics[width=0.4\textwidth]{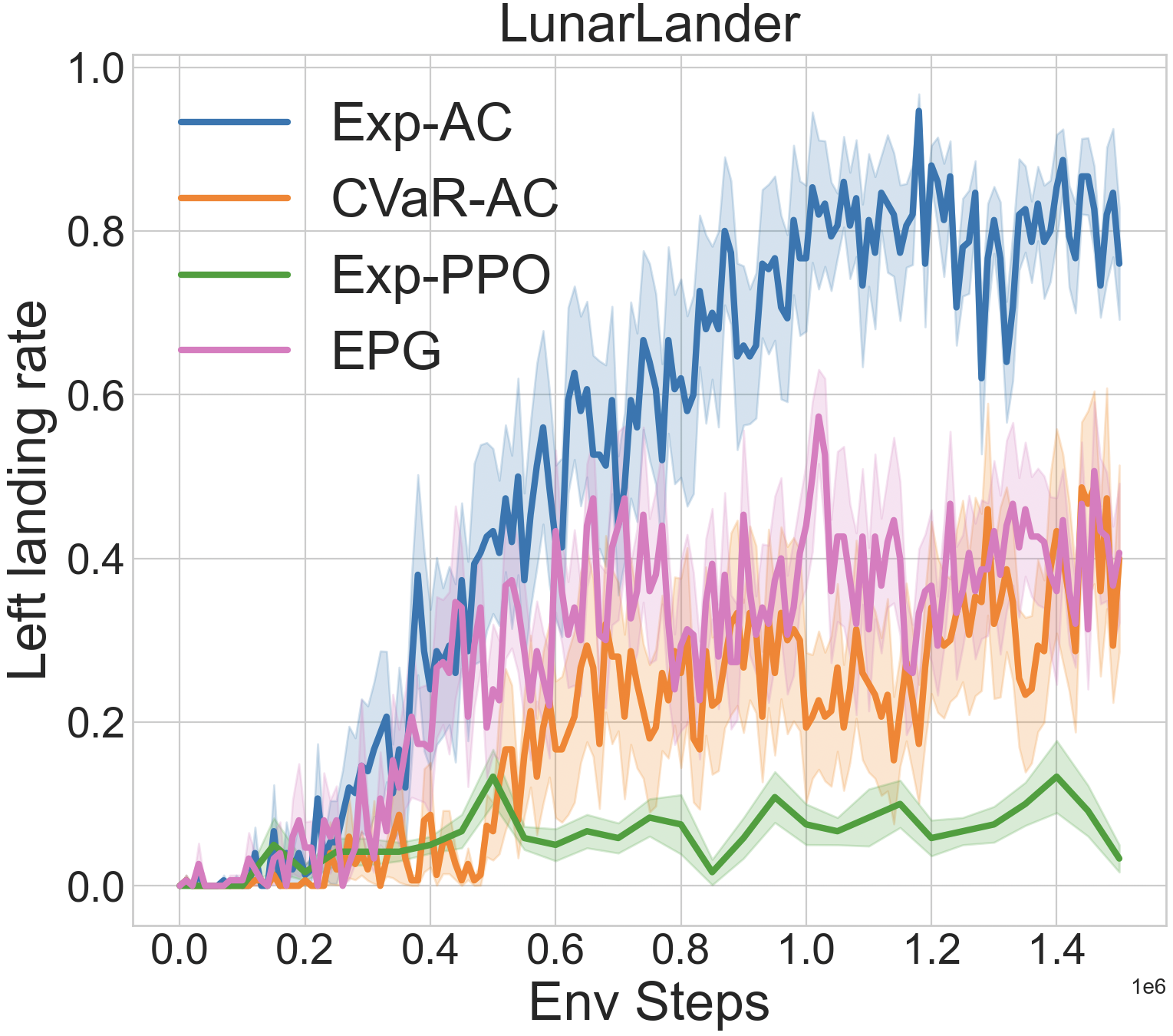}
    \end{center}
    \caption{Left landing rates in LunarLander. Curves are averaged over 10 seeds with shaded regions indicating standard errors.}
    \label{fig:ll-left}
\end{figure}

The learning curves with more environment steps for Exp-PPO in Inverted Pendulum are shown in Fig.~\ref{fig:exp-ppo-res}.

\begin{figure}
    \centering
    \begin{subfigure}[t]{0.31\textwidth}
        \centering
        \includegraphics[width=\linewidth]{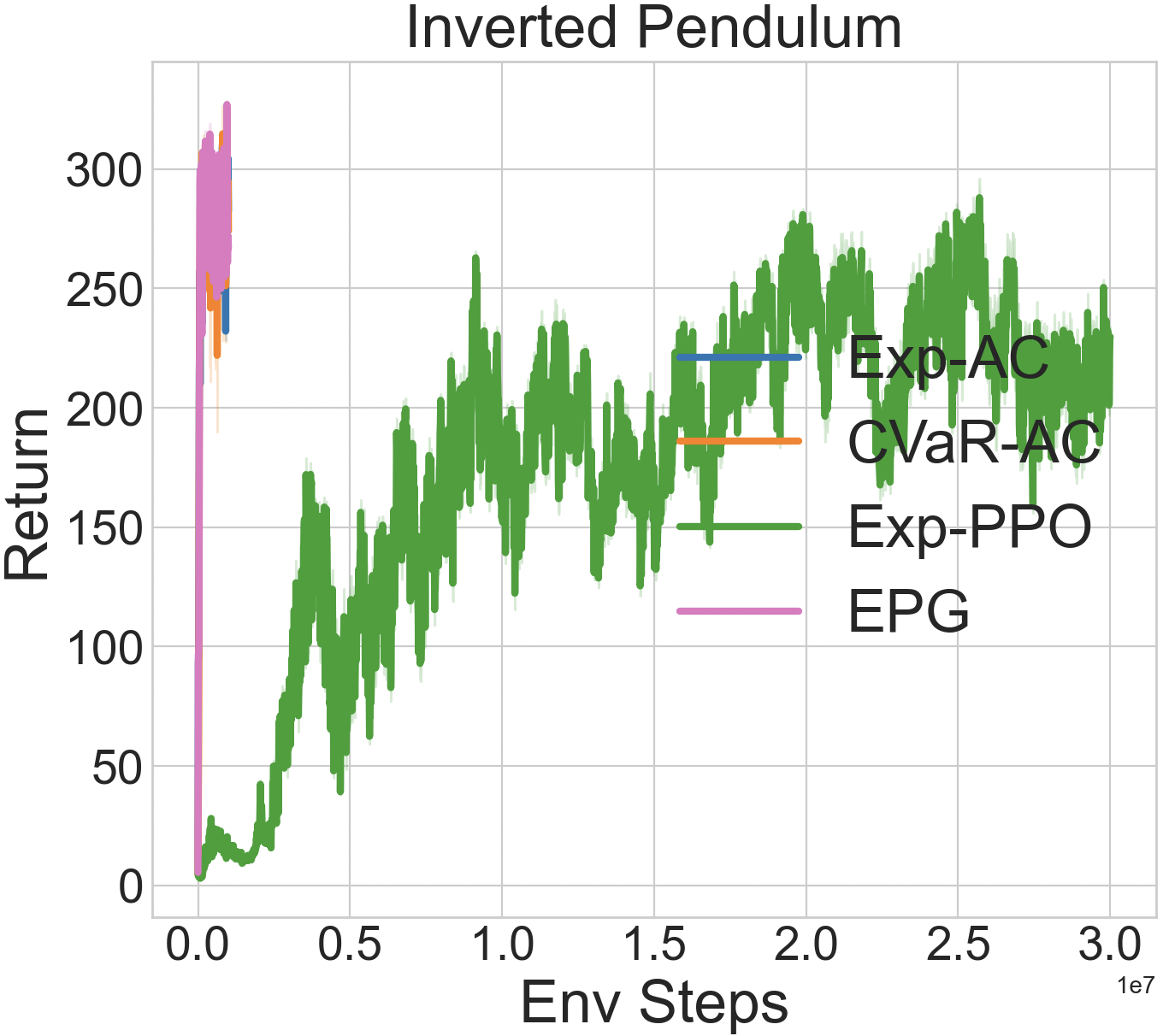}
        \caption{Expected return in Pendulum}
        \label{fig:sub-ppo1}
    \end{subfigure}
    \hspace{0.02\textwidth}
    \begin{subfigure}[t]{0.31\textwidth}
        \centering
        \includegraphics[width=\linewidth]{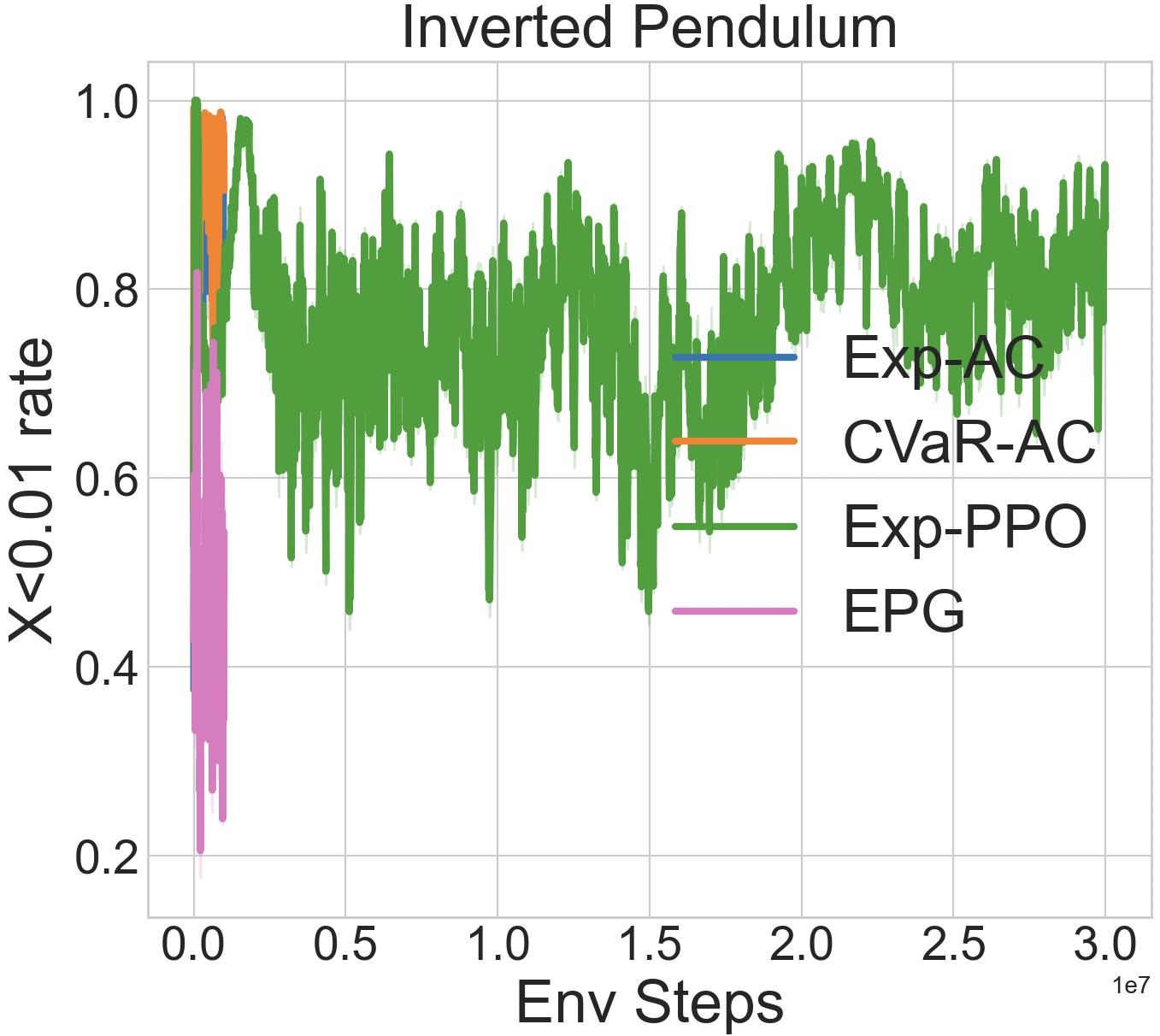}
        \caption{Risk-averse rate in Pendulum}
        \label{fig:sub-ppo2}
    \end{subfigure}
    \hspace{0.02\textwidth}
    \begin{subfigure}[t]{0.31\textwidth}
        \centering
        \includegraphics[width=\linewidth]{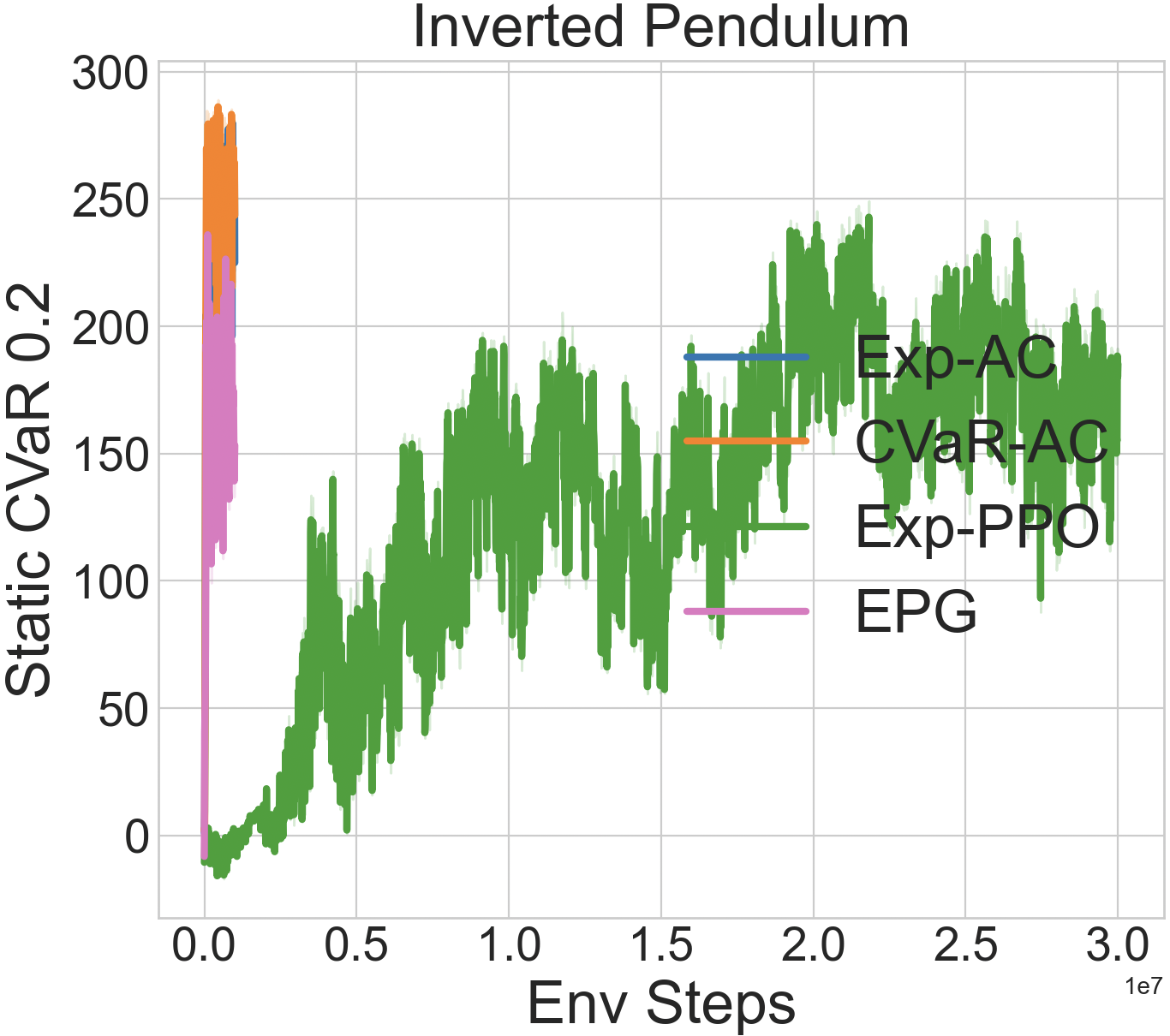}
        \caption{CVaR 0.2 of return in Pendulum}
        \label{fig:sub-ppo3}
    \end{subfigure}

    \caption{Expected return, risk-averse rate, and CVaR 0.2 of return in Inverted Pendulum. Curves are averaged over 10 seeds with shaded regions indicating standard errors.}
    \label{fig:exp-ppo-res}
\end{figure}

\begin{figure}
    \centering

    \begin{subfigure}[t]{0.3\textwidth}
        \centering
        \includegraphics[width=\linewidth]{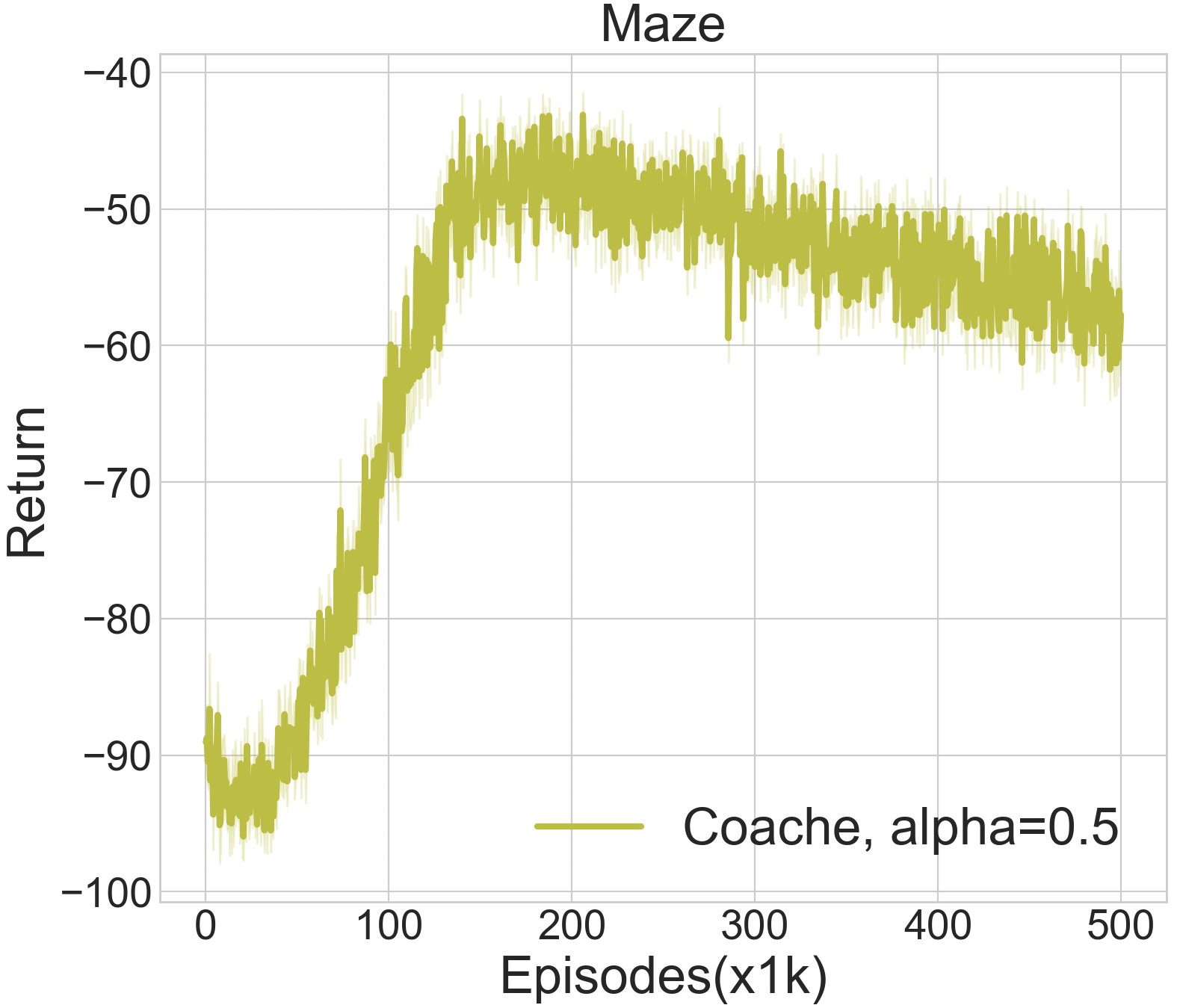}
        \caption{Expected return in maze}
        \label{fig:sub-c1}
    \end{subfigure}
    \hspace{0.05\textwidth}
    \begin{subfigure}[t]{0.3\textwidth}
        \centering
        \includegraphics[width=\linewidth]{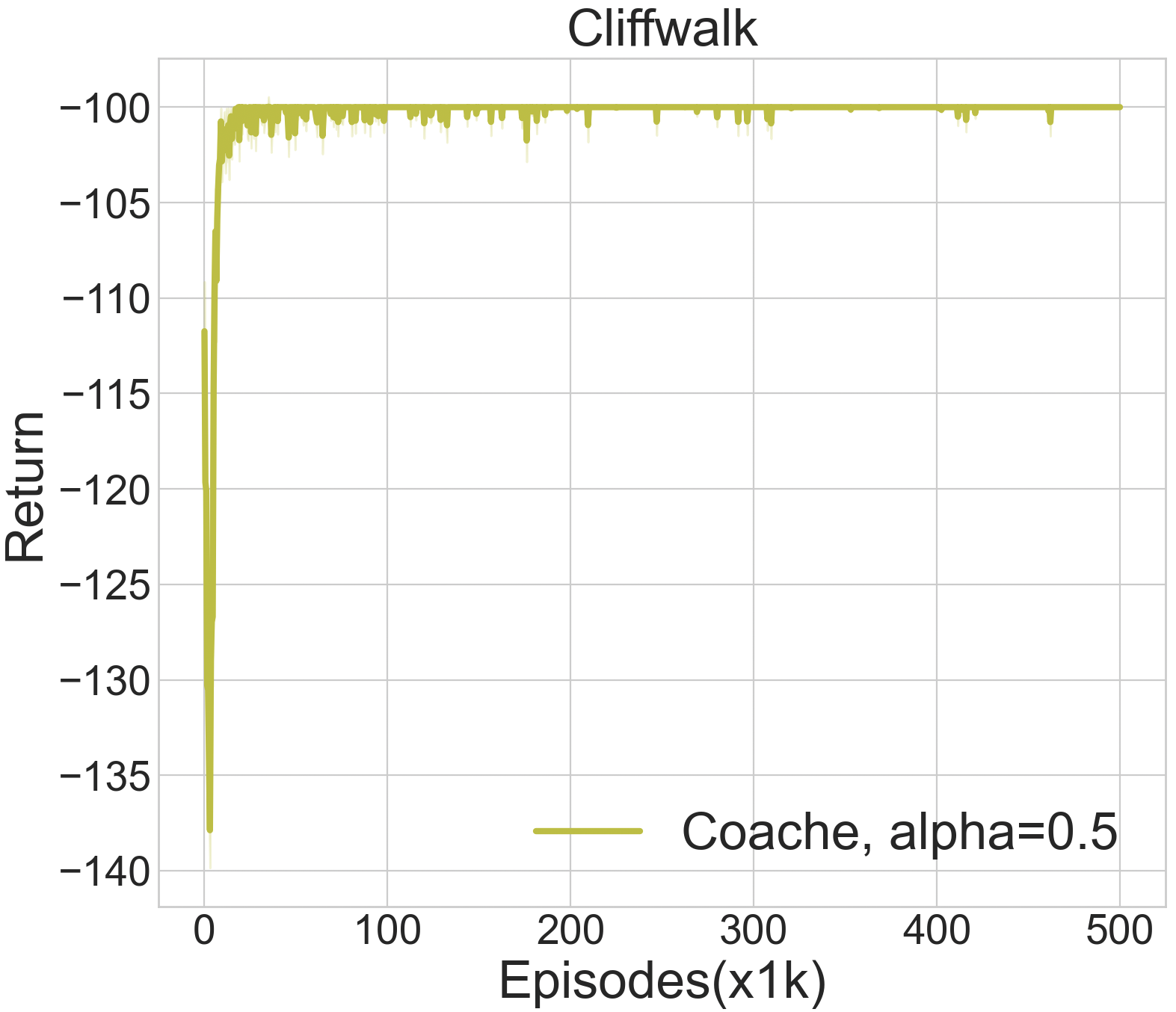}
        \caption{Expected return in cliffwalk}
        \label{fig:sub-c2}
    \end{subfigure}

    \begin{subfigure}[t]{0.3\textwidth}
        \centering
        \includegraphics[width=\linewidth]{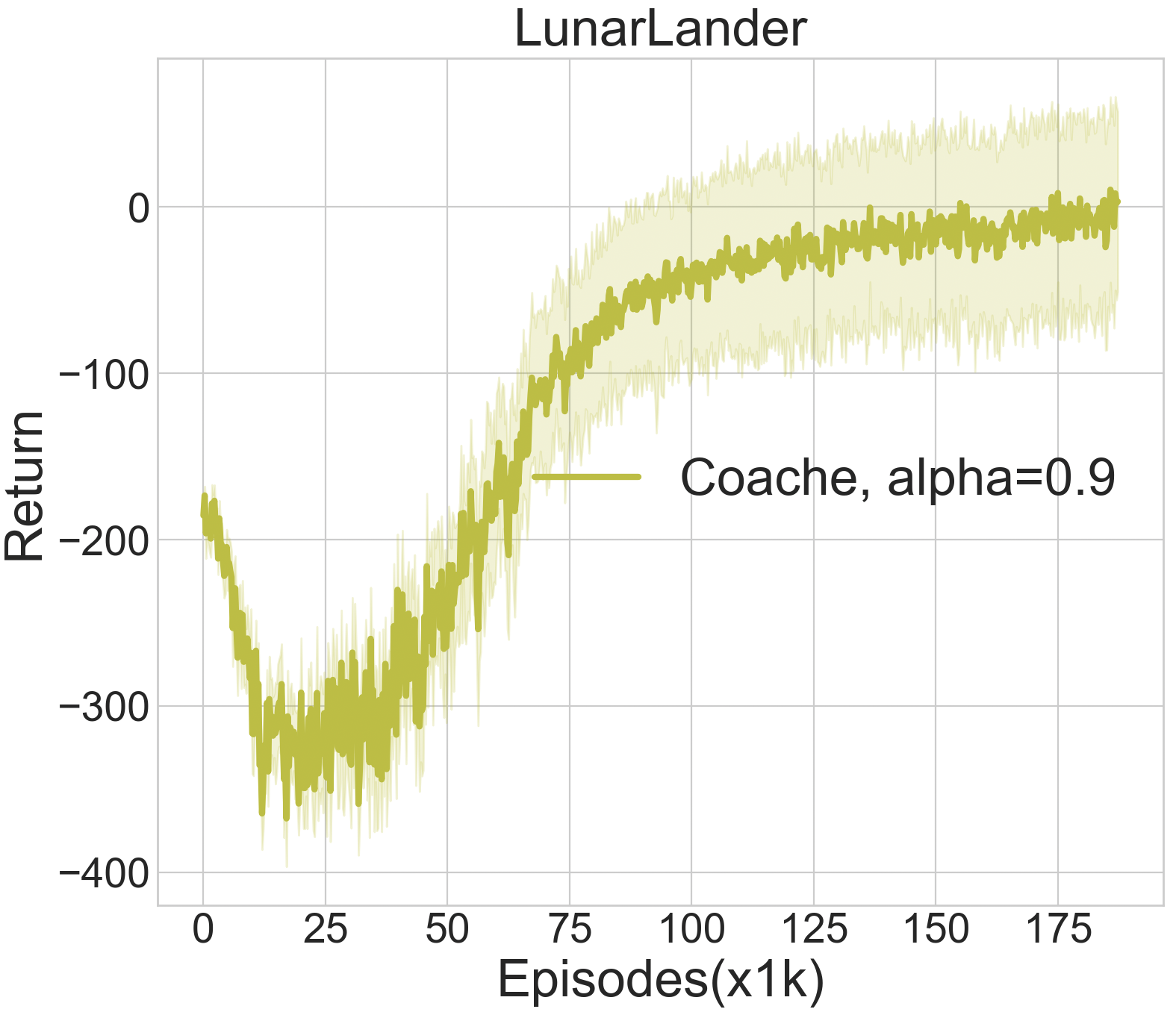}
        \caption{Expected return in LunarLander}
        \label{fig:sub-c3}
    \end{subfigure}
    \hspace{0.05\textwidth}
    \begin{subfigure}[t]{0.3\textwidth}
        \centering
        \includegraphics[width=\linewidth]{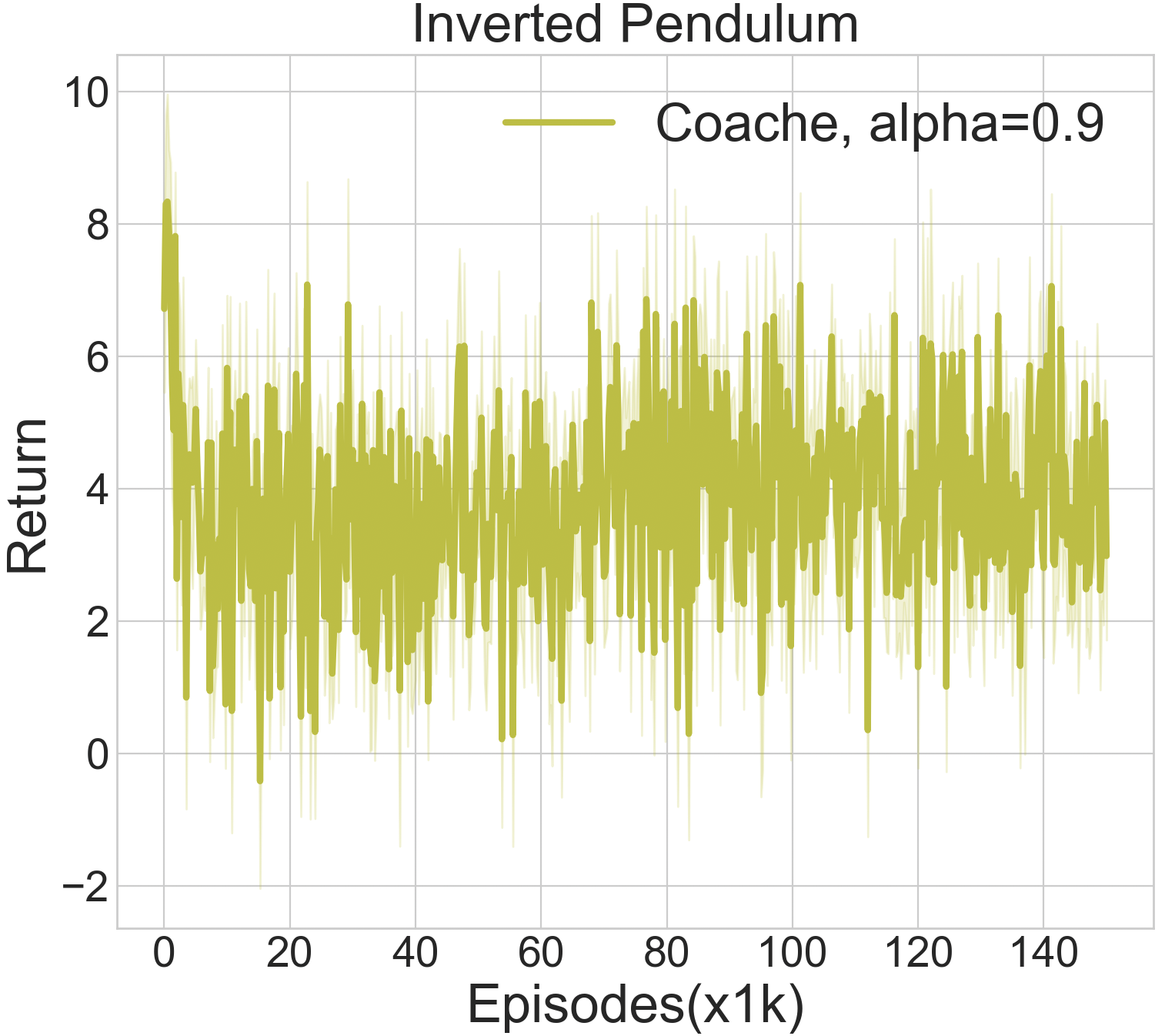}
        \caption{Expected return in Pendulum}
        \label{fig:sub-c4}
    \end{subfigure}
    \caption{Expected return of the meathod in \cite{coache2023conditionally} in four evaluation domains. Curves are averaged over 10 seeds with shaded regions indicating standard errors.}
    \label{fig:coache-res}
\end{figure}

Return curves of the method in \cite{coache2023conditionally} are shown in Fig.~\ref{fig:coache-res}. The learning rates selected are summarized in Table~\ref{tab:coache-lr}.

\begin{table}
  \caption{Learning parameters of \cite{coache2023conditionally} in four domains.}
  \label{tab:coache-lr}
  \centering
  \begin{tabular}{lccc}
    \toprule
    Env     & $\pi$ lr        & $V$ lr & quantile lr \\
    \midrule
    Maze    & 1e-4  & 2e-4 & 1e-3 \\
    Cliffwalk  & 1e-4 & 2e-4   &  1e-3 \\
    LunarLander    & 3e-5       & 6e-5  & 3e-4\\
    Pendulum    &   5e-5 & 1e-4 &5e-4\\
    \bottomrule
  \end{tabular}
\end{table}


\newpage
\onecolumn

\setlength{\bibsep}{6pt plus 0.3ex}
\bibliographystyle{apalike}
\bibliography{reference}

@article{ruszczynski2010risk,
  title={Risk-averse dynamic programming for Markov decision processes},
  author={Ruszczy{\'n}ski, Andrzej},
  journal={Mathematical programming},
  volume={125},
  number={2},
  pages={235--261},
  year={2010},
  publisher={Springer}
}

@inproceedings{dabney2018distributional,
  title={Distributional reinforcement learning with quantile regression},
  author={Dabney, Will and Rowland, Mark and Bellemare, Marc and Munos, R{\'e}mi},
  booktitle={Proceedings of the AAAI conference on artificial intelligence (AAAI)},
  volume={32},
  year={2018}
}

@article{greenberg2022efficient,
  title={Efficient risk-averse reinforcement learning},
  author={Greenberg, Ido and Chow, Yinlam and Ghavamzadeh, Mohammad and Mannor, Shie},
  journal={Advances in Neural Information Processing Systems (NeurIPS)},
  volume={35},
  pages={32639--32652},
  year={2022}
}

@article{lim2022distributional,
  title={Distributional Reinforcement Learning for Risk-Sensitive Policies},
  author={Lim, Shiau Hong and Malik, Ilyas},
  journal={Advances in Neural Information Processing Systems (NeurIPS)},
  volume={35},
  pages={30977--30989},
  year={2022}
}

@article{artzner1999coherent,
  title={Coherent measures of risk},
  author={Artzner, Philippe and Delbaen, Freddy and Eber, Jean-Marc and Heath, David},
  journal={Mathematical finance},
  volume={9},
  number={3},
  pages={203--228},
  year={1999},
  publisher={Wiley Online Library}
}

@article{benyamine2026dynamic,
  title={Dynamic Programming for Epistemic Uncertainty in Markov Decision Processes},
  author={Benyamine, Axel and Grand-Cl{\'e}ment, Julien and Petrik, Marek and Jordan, Michael I and Durmus, Alain},
  journal={arXiv preprint arXiv:2602.03381},
  year={2026}
}

@article{tamar2015policy,
  title={Policy gradient for coherent risk measures},
  author={Tamar, Aviv and Chow, Yinlam and Ghavamzadeh, Mohammad and Mannor, Shie},
  journal={Advances in Neural Information Processing Systems (NeurIPS)},
  volume={28},
  year={2015}
}

@inproceedings{
zhang2024soft,
title={Soft Robust {MDP}s and Risk-Sensitive {MDP}s: Equivalence, Policy Gradient, and Sample Complexity},
author={Runyu Zhang and Yang Hu and Na Li},
booktitle={The Twelfth International Conference on Learning Representations (ICLR)},
year={2024}
}

@article{rigter2023one,
  title={One risk to rule them all: A risk-sensitive perspective on model-based offline reinforcement learning},
  author={Rigter, Marc and Lacerda, Bruno and Hawes, Nick},
  journal={Advances in Neural Information Processing Systems (NeurIPS)},
  volume={36},
  pages={77520--77545},
  year={2023}
}

@article{marzban2023deep,
  title={Deep reinforcement learning for option pricing and hedging under dynamic expectile risk measures},
  author={Marzban, Saeed and Delage, Erick and Li, Jonathan Yu-Meng},
  journal={Quantitative finance},
  volume={23},
  number={10},
  pages={1411--1430},
  year={2023},
  publisher={Taylor \& Francis}
}

@inproceedings{silver2014deterministic,
  title={Deterministic policy gradient algorithms},
  author={Silver, David and Lever, Guy and Heess, Nicolas and Degris, Thomas and Wierstra, Daan and Riedmiller, Martin},
  booktitle={Proceedings of the International Conference on Machine Learning (ICML)},
  pages={387--395},
  year={2014},
  organization={Pmlr}
}

@book{sutton1998reinforcement,
  title={Reinforcement learning: An introduction},
  author={Sutton, Richard S and Barto, Andrew G and others},
  volume={1},
  year={1998},
  publisher={MIT press Cambridge}
}

@article{ciosek2020expected,
  title={Expected policy gradients for reinforcement learning},
  author={Ciosek, Kamil and Whiteson, Shimon},
  journal={Journal of Machine Learning Research},
  volume={21},
  number={52},
  pages={1--51},
  year={2020}
}

@article{coache2023conditionally,
  title={Conditionally elicitable dynamic risk measures for deep reinforcement learning},
  author={Coache, Anthony and Jaimungal, Sebastian and Cartea, Alvaro},
  journal={SIAM Journal on Financial Mathematics},
  volume={14},
  number={4},
  pages={1249--1289},
  year={2023},
  publisher={SIAM}
}

@article{bellini2017risk,
  title={Risk management with expectiles},
  author={Bellini, Fabio and Di Bernardino, Elena},
  journal={The European Journal of Finance},
  volume={23},
  number={6},
  pages={487--506},
  year={2017},
  publisher={Taylor \& Francis}
}

@article{rockafellar2000optimization,
  title={Optimization of conditional value-at-risk},
  author={Rockafellar, R Tyrrell and Uryasev, Stanislav and others},
  journal={Journal of risk},
  volume={2},
  pages={21--42},
  year={2000},
  publisher={Citeseer}
}

@article{huang2021convergence,
  title={On the convergence and optimality of policy gradient for markov coherent risk},
  author={Huang, Audrey and Leqi, Liu and Lipton, Zachary C and Azizzadenesheli, Kamyar},
  journal={arXiv preprint arXiv:2103.02827},
  year={2021}
}

@article{tamar2016sequential,
  title={Sequential decision making with coherent risk},
  author={Tamar, Aviv and Chow, Yinlam and Ghavamzadeh, Mohammad and Mannor, Shie},
  journal={IEEE transactions on automatic control},
  volume={62},
  number={7},
  pages={3323--3338},
  year={2016},
  publisher={IEEE}
}

@inproceedings{yu2023global,
  title={On the global convergence of risk-averse policy gradient methods with expected conditional risk measures},
  author={Yu, Xian and Ying, Lei},
  booktitle={Proceedings of the International Conference on Machine Learning (ICML)},
  pages={40425--40451},
  year={2023},
  organization={PMLR}
}

@article{shen2013risk,
  title={Risk-sensitive Markov control processes},
  author={Shen, Yun and Stannat, Wilhelm and Obermayer, Klaus},
  journal={SIAM Journal on Control and Optimization},
  volume={51},
  number={5},
  pages={3652--3672},
  year={2013},
  publisher={SIAM}
}

@article{agarwal2021theory,
  title={On the theory of policy gradient methods: Optimality, approximation, and distribution shift},
  author={Agarwal, Alekh and Kakade, Sham M and Lee, Jason D and Mahajan, Gaurav},
  journal={Journal of Machine Learning Research},
  volume={22},
  number={98},
  pages={1--76},
  year={2021}
}

@book{beck2017first,
  title={First-order methods in optimization},
  author={Beck, Amir},
  year={2017},
  publisher={SIAM}
}

@article{stein1981estimation,
  title={Estimation of the mean of a multivariate normal distribution},
  author={Stein, Charles M},
  journal={The annals of Statistics},
  pages={1135--1151},
  year={1981},
  publisher={JSTOR}
}

@article{bellini2015elicitable,
  title={On elicitable risk measures},
  author={Bellini, Fabio and Bignozzi, Valeria},
  journal={Quantitative Finance},
  volume={15},
  number={5},
  pages={725--733},
  year={2015},
  publisher={Taylor \& Francis}
}

@inproceedings{jiang2024learning,
  title={Learning diverse risk preferences in population-based self-play},
  author={Jiang, Yuhua and Liu, Qihan and Ma, Xiaoteng and Li, Chenghao and Yang, Yiqin and Yang, Jun and Liang, Bin and Zhao, Qianchuan},
  booktitle={Proceedings of the aaai conference on artificial intelligence},
  volume={38},
  number={11},
  pages={12910--12918},
  year={2024}
}

@article{fissler2016higher,
  title={Higher order elicitability and Osband's principle},
  author={Fissler, Tobias and Ziegel, Johanna F},
  journal={The Annals of Statistics},
  pages={1680--1707},
  year={2016},
  publisher={JSTOR}
}

@article{schulman2017proximal,
  title={Proximal policy optimization algorithms},
  author={Schulman, John and Wolski, Filip and Dhariwal, Prafulla and Radford, Alec and Klimov, Oleg},
  journal={arXiv preprint arXiv:1707.06347},
  year={2017}
}

@article{mnih2015human,
  title={Human-level control through deep reinforcement learning},
  author={Mnih, Volodymyr and Kavukcuoglu, Koray and Silver, David and Rusu, Andrei A and Veness, Joel and Bellemare, Marc G and Graves, Alex and Riedmiller, Martin and Fidjeland, Andreas K and Ostrovski, Georg and others},
  journal={nature},
  volume={518},
  number={7540},
  pages={529--533},
  year={2015},
  publisher={Nature Publishing Group}
}

@book{bertsekas1996neuro,
  title={Neuro-Dynamic Programming},
  author={Bertsekas, Dimitri P and Tsitsiklis, John N},
  volume={1},
  year={1996},
  publisher={Athena Scientific}
}

@article{shen2014risk,
  title={Risk-sensitive reinforcement learning},
  author={Shen, Yun and Tobia, Michael J and Sommer, Tobias and Obermayer, Klaus},
  journal={Neural computation},
  volume={26},
  number={7},
  pages={1298--1328},
  year={2014},
  publisher={MIT Press}
}

@book{borkar2023stochastic,
  title={Stochastic approximation: a dynamical systems viewpoint, Second Edition},
  author={Borkar, Vivek S},
  volume={100},
  year={2023},
  publisher={Springer}
}

@article{luo2024simple,
  title={A simple mixture policy parameterization for improving sample efficiency of cvar optimization},
  author={Luo, Yudong and Pan, Yangchen and Wang, Han and Torr, Philip and Poupart, Pascal},
  journal={Reinforcement Learning Journal},
  year={2024}
}

@inproceedings{mead2025return,
  title={Return Capping: Sample-Efficient CVaR Policy Gradient Optimisation},
  author={Mead, Harry and Costen, Clarissa and Lacerda, Bruno and Hawes, Nick},
  booktitle={Proceedings of the International Conference on Machine Learning (ICML)},
  pages={},
  year={2025}
}

@article{brockman2016openai,
  title={Openai gym},
  author={Brockman, Greg and Cheung, Vicki and Pettersson, Ludwig and Schneider, Jonas and Schulman, John and Tang, Jie and Zaremba, Wojciech},
  journal={arXiv preprint arXiv:1606.01540},
  year={2016}
}

@inproceedings{todorov2012mujoco,
  title={Mujoco: A physics engine for model-based control},
  author={Todorov, Emanuel and Erez, Tom and Tassa, Yuval},
  booktitle={2012 IEEE/RSJ international conference on intelligent robots and systems},
  pages={5026--5033},
  year={2012},
  organization={IEEE}
}

@article{schulman2015high,
  title={High-dimensional continuous control using generalized advantage estimation},
  author={Schulman, John and Moritz, Philipp and Levine, Sergey and Jordan, Michael and Abbeel, Pieter},
  journal={arXiv preprint arXiv:1506.02438},
  year={2015}
}

@article{bubeck2015convex,
  title={Convex optimization: Algorithms and complexity},
  author={Bubeck, S{\'e}bastien},
  journal={Foundations and trends in Machine Learning},
  volume={8},
  number={3-4},
  pages={231--357},
  year={2015},
  publisher={Emerald Publishing limited}
}

\end{document}